\documentclass[]{interact}

\usepackage[natbibapa,nodoi]{apacite}
\theoremstyle{plain}
\newtheorem{theorem}{Theorem}[section]
\newtheorem{lemma}[theorem]{Lemma}

\newtheorem{proposition}[theorem]{Proposition}
\theoremstyle{definition}

\theoremstyle{remark}
\newtheorem{remark}{Remark}

\usepackage{graphicx}
\graphicspath{{./Figures_1/}}
\usepackage{subfigure}
\usepackage{psfrag}
\usepackage{float}
\usepackage{xcolor}
\usepackage[normalem]{ulem}
\usepackage{amsmath,amsfonts}
\newtheorem{Assumption}{Assumption}

\begin{document}

\title{Kinematic control design for wheeled mobile robots with
  longitudinal and lateral slip}

\author{
\name{Thiago~B. Burghi\textsuperscript{a}
	\thanks{\textsuperscript{a}Email: tbb29@cam.ac.uk
							 (corresponding author)}, 
	Juliano~G. Iossaqui\textsuperscript{b}
	\thanks{\textsuperscript{b}Email:
	julianoiossaqui@utfpr.edu.br} and 
	Juan~F. Camino\textsuperscript{c}
	\thanks{\textsuperscript{c}Email:
	camino@fem.unicamp.br}}
\affil{\textsuperscript{a}Department of Engineering, University
of Cambridge, Cambridge, UK; \textsuperscript{b}Technological
Federal University of Paran\'a --- UTFPR, Londrina, PR, Brazil;
\textsuperscript{c}School of Mechanical Engineering, University
of Campinas ---  UNICAMP, Campinas, SP, Brazil.}
}

\maketitle

\begin{abstract}
The motion control of wheeled mobile robots at high speeds under
adverse ground conditions is a difficult task, since the robots'
wheels may be subject to different kinds of slip. This work
introduces an adaptive kinematic controller that is capable of
solving the trajectory tracking problem of a nonholonomic mobile
robot under longitudinal and lateral slip. While the controller 
can effectively compensate for the longitudinal slip, the 
lateral slip is a more involved problem to deal with, since
nonholonomic robots cannot directly produce movement in the
lateral direction. To show that the proposed controller is still
able to make the mobile robot follow a reference trajectory
under lateral and longitudinal time-varying slip, the solutions
of the robot's position and orientation error dynamics are shown
to be uniformly ultimately bounded. Numerical simulations are
presented to illustrate the robot's performance using the
proposed adaptive control law.
\end{abstract}



\section{Introduction}
In recent years, the development of control design techniques
for mobile robots has received considerable attention. Research
in this
area is justified by the increasing relevance of mobile robotics in
socioeconomic activities such as forestry, mining, agriculture, search
and rescue, medicine, housework, industry and space exploration
\citep{Nourbakhsh:2004:IAM,Klancar:2017:WMR}.  The efficient
deployment of mobile robots in these activities requires a solution
for the robot motion control problem, which in itself presents
interesting theoretical challenges
\citep{Morin:2006:CTT,Wu:2009:RAC,Urakubo:2015:FSN,Ibrahim:2019:PFA}.
Mobile robot controllers must deal not only with fundamental
nonholonomic constraints in the robot model, but also with practical
issues such as actuator saturation, sensor measurement noise, and the
wheel slip phenomenon.

In this context, reference tracking for nonholonomic
differential-drive mobile robots subject to wheel slip is an important
control problem.  In the ideal scenario in which the slip does not
occur, this is a well-known problem
\citep{Morin:2006:CTT,Wu:2009:RAC,Zhai:2019:ASM}.  In particular,
controller designs for reference tracking based on Lyapunov theory
have been successful in tackling the motion problem in the absence of
wheel slip
\citep{Fukao:2000:ATC,Jarzebowska:2008:APM,Panahandeh:2019:KLB}.
However, the performance of such controllers significantly
deteriorates whenever wheel slippage occurs, leading in many cases to
system instability. One solution to this issue is to explicitly
incorporate the slip effects into the controller design.

Motion control design methods \citep{Fierro:1997:CNM,Fukao:2000:ATC}
for wheeled mobile robots can be based on either kinetic or kinematic
models \citep{Wang:2008:MAS,Sidek:2008:DMC}. Wheel slip, which can be
longitudinal and lateral, is treated differently in each of these
modeling paradigms. In kinetic models, which relate robot pose
accelerations to wheel torques, the modeling of the slip phenomenon is
in general complex, as it must deal with factors such as wheel
temperature, thread pattern, and camber angle. One approach to 
overcome these issues is to employ empirical or semi-empirical methods
to model traction forces. A simpler approach that avoids the
complexities of robot kinetics is to work directly with kinematic
models, which relate robot pose speeds to wheel angular velocities.
In this case, the slip is taken into account by modeling it as a
perturbation in the wheel velocities
\cite{Wang:2008:MAS}.

When approaching the kinematic reference tracking control problem, one
is faced with the question of how to precisely measure the wheel slip
so that it can be incorporated into the control law.  A popular design
strategy consists in uncoupling the problem of reference tracking from
that of measuring or estimating the slip
\citep{Zhou:2007:UBE,Ward:2008:DMB,Moosavian:2008:ESE,Michalek:2009:TTM,Cui:2014:ATC,Chen:2019:DAT,Gao:2014:AMC}. 
In this strategy, a reference tracking controller that depends
explicitly on the slip perturbation is first designed.  Then, an
estimator or observer is independently designed to provide an estimate
of the slip.
Since the tracking and estimation problems are
posed separately, a closed-loop stability analysis is not in general
pursued; instead, performance and stability are assessed
experimentally.
This was done by
\cite{Zhou:2007:UBE,Ward:2008:DMB,Michalek:2009:TTM} using
 variants of the Kalman filter, and by \cite{Cui:2014:ATC} using a
sliding mode observer. In \cite{Chen:2019:DAT}, a
disturbance-observer-based control law is proposed, but it is built on
an auxiliary variable that depends on the slip, which cannot be
directly measured.

In contrast with this uncoupled design strategy, the papers
\cite{Iossaqui:2011:NCD,Iossaqui:2013:WRS,Li:2016:ACU} 
introduce adaptive Lyapunov-based  kinematic
control laws which are proven to asymptotically stabilize the 
closed-loop robot pose error dynamics for the particular case of
constant longitudinal slip and no lateral slip. 
While these works provide a stability analysis of the
problem, they neglect the possibility that slip perturbations are
time-varying. For the particular case of time-varying lateral slip,
but neglecting longitudinal wheel slip, this issue has recently been
given attention by \cite{Ryu:2010:DFB}, who designed a differential
flatness-based controller based on the kinetic model, 
\cite{Thomas:2019:FTP}, who designed a kinematic sliding mode
controller, and \cite{Shafiei:2019:DRT}, who designed a kinematic
$H_\infty$ controller. In these works, the controllers are proven to
drive the tracking error to a neighborhood of the origin.

This work now significantly extends the previous conference paper
\citep{Iossaqui:2013:WRS} by considering both longitudinal and lateral
slip, by allowing the slip to be time-varying, and by proving
ultimate boundedness of slip estimates for a large class of time-varying
reference trajectories. To our knowledge, this is the first time the
closed-loop stability properties of a smooth kinematic controller for
differential-drive robots is analyzed in all this generality. Our main
result proves that, as long as the reference trajectory and
longitudinal slip are slowly varying, and the lateral slip is small,
the proposed adaptive kinematic control law, based only on the robot
pose, ensures that the tracking error and the unknown longitudinal
slip estimation error are ultimately bounded by a small bound.
Important advantages of the proposed control law are its smoothness
and simplicity; the former avoids issues such as chattering while the
latter ensures it is easily implemented in practice. Our approach
relies on the analysis of the stability properties of the closed-loop
system with respect to nonvanishing perturbations caused by time
variations in the longitudinal and lateral slip, and can be readily
extended to other controller designs found in the literature.

This paper is organized as follows. Section~\ref{sec:kinematic_model}
presents the kinematic differential equations that describe the motion
of a wheeled mobile robot with longitudinal and lateral wheel slip.
Section~\ref{sec:adaptive_control_strategy} introduces the adaptive
tracking controller that is able to compensate for the slip with a
proof of closed-loop stability. Section~\ref{sec:numerical_results}
presents numerical results using the proposed adaptive controllers.

\section{Kinematic model of a wheeled mobile robot}\label{sec:kinematic_model}
This section presents the kinematic model of the wheeled mobile
robot, shown in Figure~\ref{fig:frames_robot}, with longitudinal
and lateral slip
\citep{Zhou:2007:UBE,Gonzales:2009:ACM,Iossaqui:2011:NCD}.
It is assumed that the robot can be represented by a rigid body with
two independently actuated wheels. The pose of the robot is described
by its position $(x_{\rm{p}},y_{\rm{p}})$ and its orientation
$\theta_{\rm{p}}$ in an inertial coordinate frame $F_0(x_0,y_0)$. The
robot position is given by the coordinate of its geometric center $O$,
which is also the origin of the local coordinate frame
$F_1(x_1,y_1)$. The robot orientation is the angle between the
$x_0$-axis and the $x_1$-axis. The spacing between the centerlines of
the two wheels is $b$. The motion of the robot is described by the
rotational velocity $\omega=\dot{\theta}_{\rm{p}}$ and the
translational velocity $v$. The translational velocity can be
decomposed into the forward velocity $v_x$ in the $x_1$-axis and the
lateral velocity $v_y$ in the $y_1$-axis. The angle between $v$ and
$v_x$ is $\alpha$, which is zero in the absence of lateral slip.

\begin{figure}[ht]
  \centering
  \begin{psfrags}
    \psfrag{xw}[b][b]{$x_0$}
    \psfrag{yw}[b][b]{$y_0$}
    \psfrag{xm}[b][b]{$x_1$}
    \psfrag{ym}[b][b]{$y_1$}
    \psfrag{x}[b][b]{$x_{\rm{p}}$}
    \psfrag{y}[b][b]{$y_{\rm{p}}$}
    \psfrag{z}[b][b]{$\theta_{\rm{p}}$}
    \psfrag{a}[b][b]{\footnotesize $\alpha$}
    \psfrag{b}[b][b]{$b$}
    \psfrag{w}[b][b]{$\omega$}
    \psfrag{v}[b][b]{$v$}
    \psfrag{vx}[b][b]{\footnotesize $v_x$}
    \psfrag{vy}[b][b]{\footnotesize $v_y$}
    \psfrag{o}[b][b]{$O$}
    \includegraphics[width=9cm]{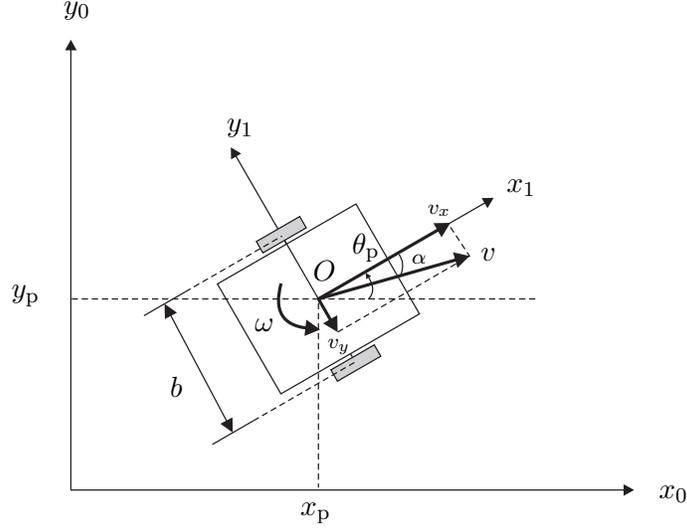}
    \caption{Schematic model of a wheeled mobile robot.\label{fig:frames_robot}}
  \end{psfrags}
\end{figure}

In the absence of any slip, the motion of the wheeled mobile robot can
be described by the kinematic model
\begin{equation}\label{eq:kin_matriz_robot_wslip}
  \dot{q}_{\rm{p}}(t) = S(q_{\rm{p}}(t)) \eta(t)
\end{equation}
where the robot pose $q_{\rm{p}}(t)$, the velocity vector $\eta(t)$,
and the matrix $S(q_{\rm{p}}(t))$, are respectively given by
\begin{equation}\label{eq:matrix_robot}
  q_{\rm{p}}(t) =
  \begin{pmatrix}
    x_{\rm{p}}(t) \\
    y_{\rm{p}}(t) \\
    \theta_{\rm{p}}(t)
  \end{pmatrix}
  , \qquad
  S(q_{\rm{p}}(t))  =
  \begin{pmatrix}
    \cos \theta_{\rm{p}}(t) & 0\\
    \sin \theta_{\rm{p}}(t) & 0\\
    0 & 1
  \end{pmatrix}
  , \qquad
  \eta(t) =
  \begin{pmatrix}
    v_x(t) \\
    \omega(t)
  \end{pmatrix}
\end{equation}

Considering now the lateral slip, the kinematic model becomes
\begin{equation}\label{eq:kin_matriz_robot_sideslip}
  \dot{q}_{\rm{p}}(t) = S_a(q_{\rm{p}}(t)) \eta(t)
\end{equation}
where the matrix $S_a(q_{\rm{p}}(t))$ is given by
\begin{equation*}
  S_a(q_{\rm{p}}(t)) =
  \begin{pmatrix}
    \left( \cos \theta_{\rm{p}}(t)
    + \sigma(t) \sin \theta_{\rm{p}}(t)\right) & \quad 0 \\
    \left( \sin \theta_{\rm{p}}(t)
    - \sigma(t) \cos \theta_{\rm{p}}(t) \right) & \quad 0 \\
    0 & \quad 1
  \end{pmatrix}
\end{equation*}
and $\sigma(t):=\tan \alpha(t) = v_y(t)/v_x(t)$ denotes the lateral
slip.

Although frequently exploited for control design, the robot velocities
$v_x(t)$ and $\omega(t)$ cannot be directly used to control a real
robot. It is more realistic to define the control input in terms of
the angular velocities $\omega_l(t)$ and $\omega_r(t)$ of the left and
right wheels, respectively. To relate these two different types of
velocities, $v_x(t)$ and $\omega(t)$ are first rewritten as
\begin{equation*}
  v_x(t) = \displaystyle \frac{v_l(t)+v_r(t)}{2}
  \quad \mbox{and} \quad
  \omega(t) = \displaystyle \frac{v_r(t)-v_l(t)}{b}
\end{equation*}
where $v_l(t)$ and $v_r(t)$ are, respectively, the linear velocities
of the left and right wheels.

When longitudinal slip is taken into consideration, the linear
velocities $v_l(t)$ and $v_r(t)$ are given by
\begin{equation*}
  \begin{split}
    v_l(t) & = r\omega_l(t)(1-i_l(t)), \qquad 0 \le i_l(t) < 1
    \\
    v_r(t) & = r\omega_r(t)(1-i_r(t)), \qquad 0 \le i_r(t) < 1
  \end{split}
\end{equation*}
where $r$ is the radius of the wheel and $i_l(t)$ and $i_r(t)$ are,
respectively, the longitudinal slip factors of the left and right
wheels.  When the slip factor is zero, the wheel rolls without slipping
(pure rolling). On the other hand, when the slip factor is unity, a
complete slipping of the wheel occurs. Defining the longitudinal slip
parameters as
\begin{equation*}
  a_l(t) = \frac{1}{1-i_l(t)}
  \quad \mbox{and} \quad
  a_r(t)= \frac{1}{1-i_r(t)}
\end{equation*}
the following relation between the robot velocities is obtained:
\begin{equation}\label{eq:velocity_robot_center_with_slip}
  \begin{split}
    v_x(t) &= 
    \frac{r}{2} \left( \frac{\omega_l(t)}{a_l(t)} +
    \frac{\omega_r(t)}{a_r(t)} \right)    
    \\
    \omega(t) &= 
    \frac{r}{b} \left( \frac{\omega_r(t)}{a_r(t)} -
    \frac{\omega_l(t)}{a_l(t)} \right)
  \end{split}
\end{equation}
Collecting the wheel angular velocities in the vector $\xi(t)
={(\omega_l(t),\omega_r(t))}^T$, relation
\eqref{eq:velocity_robot_center_with_slip} can also be expressed as
\begin{equation}
  \label{eq:eta_xi_relation}
  \eta(t)=\Phi(t)\xi(t)
\end{equation}
with matrix $\Phi(t)$ given by
\begin{equation}
  \label{eq:slip_matrix} 
  \Phi(t) = \frac{r}{2b}
  \begin{pmatrix}
    b a_l^{-1}(t)& b a_r^{-1}(t)
    \\
    -2a_l^{-1}(t) & 2a_r^{-1}(t)
  \end{pmatrix}
\end{equation}

Finally, substituting \eqref{eq:eta_xi_relation} into
\eqref{eq:kin_matriz_robot_wslip} yields the model
\begin{equation}\label{eq:matricial_kinematic_logitudinal_slip}
  \dot{q}_{\rm{p}}(t)=S(q_{\rm{p}}(t))\Phi(t)\xi(t)
\end{equation}
which takes into account how longitudinal slip affects the robot
kinematics when the lateral slip is neglected.  Analogously,
substituting
\eqref{eq:eta_xi_relation} into \eqref{eq:kin_matriz_robot_sideslip}
yields the model
\begin{equation}\label{eq:matricial_kinematic_with_slip}
  \dot{q}_{\rm{p}}(t)=S_a(q_{\rm{p}}(t))\Phi(t)\xi(t)
\end{equation}
which takes into account how both lateral and longitudinal slip
affect the robot kinematics.

\begin{remark}\label{rem:effective_control_input}
  Notice that the matrix $\Phi(t)$ in \eqref{eq:slip_matrix} is
  always nonsingular. Thus, relation \eqref{eq:eta_xi_relation}
  can be inverted, giving
  \begin{equation*}
    \xi(t) = \Phi^{-1}(t)\eta(t)  
  \end{equation*}
  or more explicitly
  \begin{equation*}
    \begin{pmatrix}
      \omega_l(t) \\
      \omega_r(t)
    \end{pmatrix} = \displaystyle \frac{1}{2r}
    \begin{pmatrix}
      2a_l(t) & -b a_l(t) \\
      2a_r(t) & b a_r(t)
    \end{pmatrix}
    \begin{pmatrix}
      v_x(t) \\
      \omega(t)
    \end{pmatrix}
  \end{equation*}
  Consequently, if the longitudinal slip parameters $a_l(t)$ and
  $a_r(t)$ are precisely known, the wheel angular velocity vector
  $\xi(t) = (\omega_l(t),\omega_r(t))^T$ can be precisely obtained
  from any desired velocity vector $\eta(t)={(v_x(t),\omega(t))}^T
  $. For this reason, $\eta(t)$ will be called the kinematic control
  input, while $\xi(t)$ will be called the effective control input.
  
\end{remark}

\section{Adaptive kinematic control law}\label{sec:adaptive_control_strategy}
This section introduces a trajectory tracking kinematic control law
for the robot shown in Figure~\ref{fig:frames_robot}.  Section
\ref{sec:preliminaries} poses the trajectory tracking control problem,
and recalls some results from the literature concerning the case in
which the slip is fully neglected.  Section
\ref{sec:adaptive_control_law} introduces our control law and shows
that it drives the robot to asymptotically track a broad class of
reference trajectories whenever the longitudinal slip is constant and
the lateral slip is neglected.  The remaining sections are dedicated
to showing that this same control law can still solve the trajectory
tracking problem when the longitudinal slip and the (non-negligible)
lateral slip are slowly time-varying. More precisely, our main result
states that the tracking error is uniformly ultimately bounded, in
other words, the reference trajectory can be tracked closely (although
not asymptotically) by the robot. Section \ref{sec:time_varying_slip}
presents the main result and its underlying assumptions, and Sections
3.3 and 3.4 provide the details involved in its proof.
						
\subsection{Preliminaries}\label{sec:preliminaries} 

The goal of the control design problem is to provide an input to the
robot such that the robot pose $q_{\rm{p}}(t)$ tracks a given
reference trajectory $q_{\rm{ref}}(t)$. This can be done by enforcing
that
\begin{equation*}
  \lim_{t\to \infty}(q_{\rm{ref}}(t) - q_{\rm{p}}(t)) = 0
\end{equation*}
where the reference trajectory
$q_{\rm{ref}}(t)={(x_{\rm{ref}}(t),y_{\rm{ref}}(t),\theta_{\rm{ref}}(t))}^T$
is generated using the kinematic model
\begin{equation}\label{eq:matrix_robot_reference}
  \dot{q}_{\rm{ref}}(t)=S(q_{\rm{ref}}(t)) \eta_{\rm{ref}}(t)
\end{equation}
with $\eta_{\rm{ref}}(t)={(v_{\rm{ref}}(t),\omega_{\rm{ref}}(t))}^T$
 the reference input and $S(\cdot)$ the map given
 by \eqref{eq:matrix_robot}.

To ensure that the robot pose $q_{\rm{p}}(t)$ will track the reference
trajectory $q_{\rm{ref}}(t)$, the pose error
$e(t)={(e_1(t),e_2(t),e_3(t))}^T$ is defined as follows:
\begin{equation}\label{eq:posture_error}
  \begin{pmatrix}
    e_1(t) \\
    e_2(t) \\
    e_3(t)
  \end{pmatrix} =
  \begin{pmatrix}
    \cos \theta_{\rm{p}}(t) & \sin \theta_{\rm{p}}(t) & 0 \\
    -\sin \theta_{\rm{p}}(t) & \cos \theta_{\rm{p}}(t) & 0 \\
    0 & 0 & 1
  \end{pmatrix}
  \begin{pmatrix}
    x_{\rm{ref}}(t)-x_{\rm{p}}(t) \\
    y_{\rm{ref}}(t)-y_{\rm{p}}(t) \\
    \theta_{\rm{ref}}(t)-\theta_{\rm{p}}(t)
  \end{pmatrix}
\end{equation}
Clearly, the robot pose $q_{\rm{p}}(t)$ converges to the reference
trajectory $q_{\rm{ref}}(t)$ whenever the pose error $e(t)$ converges
to zero. Thus, it is possible to solve the trajectory tracking problem
by showing that the origin of the pose error dynamics is an
asymptotically stable equilibrium point.

Neglecting the longitudinal and the lateral slip, the error dynamics
in terms of the kinematic control input $\eta(t)=(v_x(t),\omega(t))^T$
is directly obtained from \eqref{eq:kin_matriz_robot_wslip},
\eqref{eq:matrix_robot_reference}, and \eqref{eq:posture_error}, as
follows:
\begin{equation}\label{eq:dynamics_posture_error}
  \begin{pmatrix}
    \dot{e}_1(t) \\
    \dot{e}_2(t) \\
    \dot{e}_3(t)
  \end{pmatrix} =
  \begin{pmatrix}
  v_{\rm{ref}}(t) \cos e_3(t) \\
  v_{\rm{ref}}(t) \sin e_3(t) \\
  \omega_{\rm{ref}}(t)
  \end{pmatrix}
  +
  \begin{pmatrix}
    -1 & e_2(t)  \\
    0 & -e_1(t)  \\
    0 & -1
  \end{pmatrix}
  \eta(t)
\end{equation}

Under the hypothesis that the reference inputs $\omega_{\rm{ref}}$ and
$v_{\rm{ref}}$ are constants, the authors in \cite{Kim:1998:GAS}
showed that the kinematic control law
$\eta(t)=\eta_c(t)={(v_c(t),\omega_c(t))}^T$, with
\begin{equation}\label{eq:auxiliary_control_input_without_slip}
  \begin{split}
    v_c(t) &= v_{\rm{ref}}(t) \cos e_3(t) - k_3 e_3(t) 
    \omega_c(t) + k_1 e_1(t)
    \\
    \omega_c(t) &=  \omega_{\rm{ref}}(t)
    + \displaystyle
    \frac{v_{\rm{ref}}(t)}{2}\left[k_2\left(e_2(t)+k_3e_3(t)\right)
      +\frac{1}{k_3}\sin e_3(t)\right]
      , \qquad k_i>0
  \end{split}
\end{equation}
asymptotically stabilizes the origin of the ideal slip-free
error dynamics \eqref{eq:dynamics_posture_error}, thus ensuring that
the pose error $e(t)$ converges to zero. This fact can be shown using
the following Lyapunov function:
\begin{equation}\label{eq:lyap_fun_noslip}
  V(e(t)) = \frac{1}{2}e_1^2(t) + \frac{1}{2} 
  	{\left(e_2(t)+k_3e_3(t)
    \right)}^2 + \frac{(1-\cos e_3(t))}{k_2}
\end{equation}

Notice that it is straightforward to take into account the
longitudinal slip, while neglecting the lateral slip, by
substituting \eqref{eq:eta_xi_relation} into
\eqref{eq:dynamics_posture_error}, that is
\begin{equation}
	\label{eq:din_erro_postura_cin}
  \dot{e}(t) =
  \begin{pmatrix}
    v_{\rm{ref}}(t) \cos e_3(t) \\
    v_{\rm{ref}}(t) \sin e_3(t) \\
    \omega_{\rm{ref}}(t)    
  \end{pmatrix}
  +
  \begin{pmatrix}
    -1 & e_2(t)  \\
    0 & -e_1(t)  \\
    0 & -1
  \end{pmatrix}
  \Phi(t) \xi(t)
\end{equation}
Thus, an obvious consequence of
Remark~\ref{rem:effective_control_input}, which has not been
emphasized in the literature, is that any kinematic control law
$\eta_c(t)$ such as \eqref{eq:auxiliary_control_input_without_slip},
which stabilizes the origin of the ideal system
\eqref{eq:dynamics_posture_error}, can also be used to stabilize the
origin of the error dynamics \eqref{eq:din_erro_postura_cin}, as long
as the longitudinal slip is precisely known. For this to be
accomplished, it suffices to choose the effective control input
$\xi(t) = \Phi^{-1}(t) \eta_c(t)$.  However, when the longitudinal
slip is unknown, one can no longer make use of the matrix $\Phi(t)$ in
the control input $\xi(t)$, and this simple approach cannot be
applied.

In the next section, the kinematic control law 
\eqref{eq:auxiliary_control_input_without_slip} will be exploited in order to
derive an effective control law to stabilize the origin of \eqref{eq:din_erro_postura_cin} when the longitudinal slip is constant.

\subsection{Constant longitudinal slip without lateral slip}\label{sec:adaptive_control_law}

For the particular case in which the reference input
$\eta_{\rm{ref}}(t)={(v_{\rm{ref}}(t),\omega_{\rm{ref}}(t))}^T$ is a
piecewise continuous function of time and the longitudinal slip
parameters $a_l$ and $a_r$ are unknown constants, it is possible to
provide a controller to compensate for the slip by combining the
kinematic control law \eqref{eq:auxiliary_control_input_without_slip}
with an adaptive update rule. The key insight is to choose the
effective control input as $\xi(t) = \hat{\Phi}^{-1}(t) \eta_c(t)$,
where $\hat{\Phi}(t)$ is a matrix constructed by replacing the
parameters $a_l$ and $a_r$ in $\Phi(t)$ by their estimates
$\hat{a}_l(t)$ and $\hat{a}_r(t)$, respectively, given by
\begin{equation}\label{eq:slipconstantestimation}
  \begin{split}
    \hat{a}_l(t) &= a_l + \tilde{a}_l(t)
    \\
    \hat{a}_r(t) &= a_r + \tilde{a}_r(t)
  \end{split}
\end{equation}
with $\tilde{a}_l(t)$ and $\tilde{a}_r(t)$ the estimation errors.
This results in the effective control input
\begin{equation}\label{eq:new_effective_control_input}
  \begin{pmatrix}
    \omega_l(t) \\
    \omega_r(t)
  \end{pmatrix} = \frac{1}{2r}
  \begin{pmatrix}
    2\hat{a}_l(t) & -b\hat{a}_l(t) \\
    2\hat{a}_r(t) & b\hat{a}_r(t)
  \end{pmatrix}
  \begin{pmatrix}
    v_c(t) \\
    \omega_c(t)
  \end{pmatrix} \\
\end{equation}
with $\eta_c(t) = {(v_c(t),\omega_c(t))}^T$ given by
\eqref{eq:auxiliary_control_input_without_slip}, and the corresponding
kinematic control law $\eta(t) = \Phi(t) \hat{\Phi}^{-1}(t)
\eta_c(t)$.

It will now be shown that the adaptive control law
\eqref{eq:new_effective_control_input} asymptotically stabilizes the
origin of \eqref{eq:din_erro_postura_cin}, for constant longitudinal
slip parameters, i.e., with $\Phi(t)=\Phi$ a constant matrix. By
substituting \eqref{eq:slipconstantestimation} and
\eqref{eq:new_effective_control_input} into
\eqref{eq:din_erro_postura_cin}, the pose error dynamics after some
manipulations is given by
\begin{align}
  \dot{e}_1(t) &=\left(1+\frac{\tilde{a}_r(t)}{a_r}\right)
  \left(\frac{e_2(t)}{b}-\frac{1}{2}\right)\left(v_c(t)+
  \frac{b}{2}\omega_c(t) \right)+v_{\rm{ref}}(t)
  \cos e_3(t)
  \nonumber \\
  &\quad-\left(1+\frac{\tilde{a}_l(t)}{a_l}\right)
  \left(\frac{e_2(t)}{b}+\frac{1}{2}\right)\left(v_c(t)
  -\frac{b}{2}\omega_c(t)\right)
  \label{eq:error_derivative_e1}
  \\
  \dot{e}_2(t) &= \left(1+\frac{\tilde{a}_l(t)}{a_l}\right)
  \left(v_c(t)-\frac{b}{2}\omega_c(t)\right)\frac{e_1(t)}{b}+v_{\rm{ref}}(t)
  \sin e_3(t)
  \nonumber \\
  &\quad-\left(1+\frac{\tilde{a}_r(t)}{a_r}\right)
  \left(v_c(t)+\frac{b}{2}\omega_c(t)\right)\frac{e_1(t)}{b}
  \label{eq:error_derivative_e2}
  \\
  \dot{e}_3(t) &= \omega_{\rm{ref}}(t)-\frac{1}{b}
  \left(1+\frac{\tilde{a}_r(t)}{a_r}\right)
  \left(v_c(t)+\frac{b}{2}\omega_c(t)\right)+\frac{1}{b}\left(1
  +\frac{\tilde{a}_l(t)}{a_l}\right)
  \left(v_c(t) -\frac{b}{2}\omega_c(t)\right)
  \label{eq:error_derivative_e3}
\end{align}

To obtain an update rule for the slip estimates, the following
Lyapunov function candidate is considered
\begin{equation}\label{eq:lyap_constant_slip}
  V_a(e_a(t)) = V(e(t))
  +\frac{\tilde{a}_l^2(t)}{2 \gamma_1 a_l}
  +\frac{\tilde{a}_r^2(t)}{2 \gamma_2 a_r}
  , \qquad \gamma_i > 0
\end{equation}
with $V(e(t))$ given by \eqref{eq:lyap_fun_noslip} and the augmented
error defined as $e_a(t)={(e_1, e_2,e_3,\tilde{a}_l,
\tilde{a}_r)}^T$.
Using \eqref{eq:error_derivative_e1}, \eqref{eq:error_derivative_e2},
and \eqref{eq:error_derivative_e3} and noticing that
$\dot{\tilde{a}}_l(t) = \dot{\hat{a}}_l(t)$ and $\dot{\tilde{a}}_r(t)
= \dot{\hat{a}}_r(t)$, since $a_l$ and $a_r$ are constant, the
derivative of $V_a(e_a(t))$ can be expressed, after some manipulations, as
\begin{equation*}
  \begin{split}
    \dot{V}_a(e_a(t)) & = -k_1e_1^2(t) -
    \frac{v_{\rm{ref}}(t)}{2}k_2k_3{(e_2(t)+k_3e_3(t))}^2
    -\frac{v_{\rm{ref}}(t)}{2k_2k_3}\sin^2 e_3(t)
    \\
    & \qquad + \frac{\tilde{a}_l(t)}{a_l}
    \bigg\{\frac{\dot{\hat{a}}_l(t)}{\gamma_1}
    - \left(v_c(t)-\frac{b}{2}\omega_c(t)\right)
    \bigg[ \left(\frac{e_2(t)}{b} + \frac{1}{2} \right) e_1(t)
      \\  &\qquad
      -\left(\frac{e_1(t)}{b} +
      \frac{k_3}{b}\right)(e_2(t)+k_3e_3(t))
      -\frac{1}{b k_2}\sin e_3(t)\bigg]\bigg\}
    \\
    & \qquad + \frac{\tilde{a}_r(t)}{a_r}
    \bigg\{\frac{\dot{\hat{a}}_r(t)}{\gamma_2}
    - \left(v_c(t) + \frac{b}{2}\omega_c(t)\right)
    \bigg[-\left(\frac{e_2(t)}{b}
      - \frac{1}{2}\right) e_1(t)
      \\  &\qquad      
      + \left(\frac{e_1(t)}{b} + \frac{k_3}{b}\right) (e_2(t)+k_3e_3(t))+
      \frac{1}{b k_2}\sin e_3(t)\bigg]\bigg\}
  \end{split}
\end{equation*}
with $v_c(t)$ and $\omega_c(t)$ given by
\eqref{eq:auxiliary_control_input_without_slip}. Now, choosing
the update rule for $\hat{a}_l(t)$ and $\hat{a}_r(t)$ as
\begin{equation}\label{eq:update_rule_constant_slip}
  \begin{split}
    \dot{\hat{a}}_l(t) &=  \gamma_1
    \left(v_c(t)-\frac{b}{2}\omega_c(t)\right)\bigg[\left(\frac{e_2(t)}{b}
      +\frac{1}{2}\right)e_1(t)
      \\
      & \qquad
      - \left(\frac{e_1(t)}{b}+
      \frac{k_3}{b}\right) (e_2(t)+k_3e_3(t))
      -\frac{1}{b k_2}\sin e_3(t)\bigg]
    \\
    \dot{\hat{a}}_r(t) &=  \gamma_2
    \left(v_c(t)+\frac{b}{2}\omega_c(t)\right)\bigg[-\left(\frac{e_2(t)}{b}
      -\frac{1}{2}\right)e_1(t)
      \\
      & \qquad
      +\left(\frac{e_1(t)}{b}+\frac{k_3}{b}\right)(e_2(t)+k_3e_3(t))+
      \frac{1}{b k_2}\sin e_3(t)\bigg]
  \end{split}
\end{equation}
it finally follows that
\begin{equation}\label{eq:lyapunov_derivative_proof}
  \begin{split}
    \dot{V}_a(e_a(t))& =-k_1e_1^2(t) -
    \frac{v_{\rm{ref}}(t)}{2}k_2k_3{(e_2(t) + k_3e_3(t))}^2
    - \frac{v_{\rm{ref}}(t)}{2k_2k_3}\sin^2 e_3(t)
  \end{split}
\end{equation}

With this choice of update rule, the dynamics of the augmented error
$e_a(t)$ is thus given by
\begin{equation}\label{eq:nonautonomous_function_error_aug}
  \dot{e}_a=f_a(t,e_a(t))
\end{equation}
where the rows of $f_a$ are the expressions on the right-hand sides of
\eqref{eq:error_derivative_e1}-\eqref{eq:error_derivative_e3} and
\eqref{eq:update_rule_constant_slip}, with $\dot{\tilde{a}}_l(t) =
\dot{\hat{a}}_l(t)$ and $\dot{\tilde{a}}_r(t) = \dot{\hat{a}}_r(t)$,
and the control input $v_c(t)$ and $\omega_c(t)$ given by
\eqref{eq:auxiliary_control_input_without_slip}. At this point,
Theorem~8.4 from \citep[p. 323]{Khalil:2002:NS} can be applied to show
that the origin of \eqref{eq:nonautonomous_function_error_aug} is
uniformly asymptotically stable. To see this fact, observe that the
function $V_a(e_a(t))$ is positive definite over the domain
$D=\{e_a(t) \in \mathbb{R}^5\ | -\pi < e_3(t) <\pi\}$.  In addition,
the function $\dot{V}_a(t,e_a(t))$ is bounded above by $-W(e_a(t))$,
with $W(e_a(t))$ the positive semidefinite function
\begin{equation*}
  W(e_a(t)) = k_1e_1^2(t) + \frac{k_2k_3\mu}{2}{(e_2(t) + k_3e_3(t))}^2
  + \frac{\mu}{2k_2k_3}\sin^2 e_3(t), \qquad 0<\mu\le v_{\rm{ref}}(t)
\end{equation*}
By noticing that $f_a(t,0)=0$, all conditions of Theorem 8.4 are
satisfied, which implies that all solutions of
\eqref{eq:nonautonomous_function_error_aug} are bounded and satisfy
$W(e_a(t)) \to 0$ as $t \to \infty$. Consequently, $e(t) \to 0$ as $t
\to \infty$.
This conclusion extends a previous result shown
in \cite{Iossaqui:2011:NCD}, that dealt with the particular case in
which the reference input
$\eta_{\rm{ref}}={(v_{\rm{ref}},\omega_{\rm{ref}})}^T$ is restricted
to be constant.
  
Figure~\ref{fig:adaptive_kin} shows the schematic representation of
the closed-loop system composed of the reference trajectory, the
adaptive kinematic control law, and the robot. The numbering inside
the blocks indicates the corresponding equation number.  The adaptive
control law is composed of the kinematic control law
$\eta_c={(v_c,\omega_c)}^T$ given by
\eqref{eq:auxiliary_control_input_without_slip} together with the
update rule given by \eqref{eq:update_rule_constant_slip}.

\begin{figure}[ht]
  \footnotesize
  \centering
  \begin{psfrags}
    \psfrag{1}[b][b]{\eqref{eq:matricial_kinematic_logitudinal_slip}}
    \psfrag{2}[b][b]{\eqref{eq:matrix_robot_reference}}
    \psfrag{3}[b][b]{\eqref{eq:posture_error}}
    \psfrag{4}[b][b]{\eqref{eq:auxiliary_control_input_without_slip}}
    \psfrag{5}[b][b]{\eqref{eq:new_effective_control_input}}
    \psfrag{6}[b][b]{\eqref{eq:update_rule_constant_slip}}
    \psfrag{a}[b][b]{$\eta_{\rm{ref}}(t)$}
    \psfrag{c}[b][b]{$q_{\rm{ref}}$}
    \psfrag{d}[b][b]{$e$}
    \psfrag{e}[b][b]{$\eta_c$}
    \psfrag{f}[b][b]{$q_{\rm{p}}$}
    \psfrag{v}[b][b]{$\xi$}
    \psfrag{s}[b][b]{$\hat{a}_l$}
    \psfrag{t}[b][b]{$\hat{a}_r$}
    \psfrag{aL}[b][b]{$a_l$}
    \psfrag{aR}[b][b]{$a_r$}
    \psfrag{ref}[c][c]{\parbox{39pt}{\centering
        Reference\\Trajectory}}
    \psfrag{erro}[c][c]{\parbox{20pt}{\centering
        Pose\\Error}}
    \psfrag{entaux}[c][c]{\parbox{50pt}{\centering
      Kinematic\\Control Law}}
    \psfrag{entefe}[c][c]{\parbox{60pt}{\centering
        Effective\\Control Input}}
    \psfrag{ad}[c][c]{\parbox{27pt}{\centering
        Update\\Rule}}
    \psfrag{robo}[b][b]{Robot}
    \includegraphics[width=0.8\linewidth]{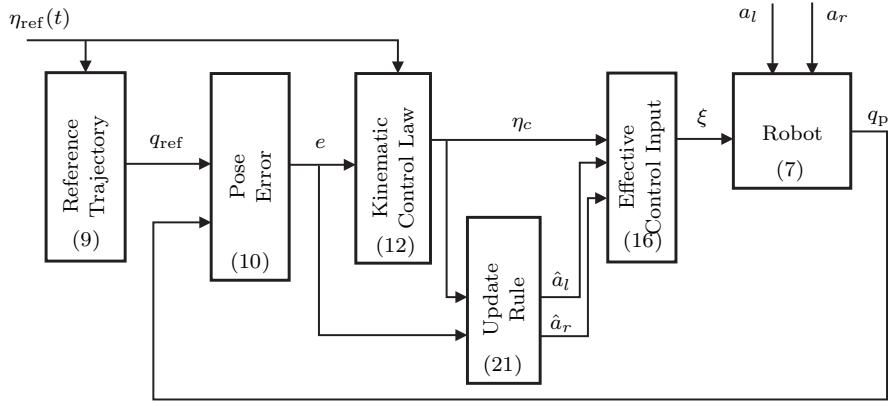}
    \caption{Adaptive kinematic control strategy.\label{fig:adaptive_kin}}
  \end{psfrags}
\end{figure}

The adaptive kinematic control law described in this section was
obtained neglecting the lateral slip and assuming constant
longitudinal slip parameters. In the next section, using a different
approach, it is shown that this same adaptive kinematic control law
can also guarantee that the augmented error signal $e_a(t)$ is
ultimately bounded for a slow time-varying reference input
$\eta_{\rm{ref}}(t)$, a sufficiently small lateral slip $\sigma(t)$,
and slow time-varying longitudinal slip parameters $a_l(t)$ and
$a_r(t)$.

\subsection{Time-varying longitudinal and lateral slip}
\label{sec:time_varying_slip} 

This section extends the stability analysis of the adaptive 
kinematic control law introduced in
Section~\ref{sec:adaptive_control_law} to the case where the
longitudinal slip and the (non-negligible) lateral slip are 
time-varying. More specifically, assuming that the time-varying  
parameters of the system vary slowly, it will be shown that the
proposed control law ensures that the augmented error 
$e_a(t)={(e_1, e_2,e_3,\tilde{a}_l, \tilde{a}_r)}^T$ is
ultimately bounded, i.e., it reaches a neighborhood of the origin
in finite time, and stays in that neighborhood thereafter.

When the longitudinal slip is time-varying, the estimates are 
given by
\begin{equation}\label{eq:sliptvestimation}
  \begin{split}
    \hat{a}_l(t) &= a_l(t) + \tilde{a}_l(t)
    \\
    \hat{a}_r(t) &= a_r(t) + \tilde{a}_r(t)
  \end{split}
\end{equation}
It will be assumed throughout that the slip parameters
$a_l(t)$ and $a_r(t)$ are differentiable functions of time.
In that case, considering time-varying slip parameters, the
dynamics of the augmented error $e_a(t)$, derived using
\eqref{eq:matricial_kinematic_with_slip},
\eqref{eq:matrix_robot_reference}, \eqref{eq:posture_error},
\eqref{eq:new_effective_control_input},  
\eqref{eq:update_rule_constant_slip},
and \eqref{eq:sliptvestimation},
can be written (after some manipulations) as the
following perturbed system:
\begin{equation}\label{eq:nonautonomous_perturbed_system}
  \dot{e}_a(t) = f_a(t,e_a(t)) + g(t,e_a(t)) + g_n(t,e_a(t))
\end{equation}
with a nominal term $f_a(t,e_a(t))$ given by
\begin{equation*}
  \hspace*{-50pt}
  f_a(t,e_a(t)) =
  \begin{pmatrix}
    \Delta_r \Omega_3 \Omega_1
    + v_{\rm{ref}}(t) \cos e_3(t)
    -\Delta_l \Omega_4 \Omega_2
    \\
    \Delta_l  \Omega_2
    \displaystyle\frac{e_1(t)}{b}+v_{\rm{ref}}(t) \sin e_3(t)
    -\Delta_r \Omega_1
    \displaystyle\frac{e_1(t)}{b}
    \\
    \omega_{\rm{ref}}(t)-\displaystyle\frac{1}{b}
    \Delta_r \Omega_1
    +\displaystyle\frac{1}{b}
    \Delta_l \Omega_2
    \\
    \gamma_1 \Omega_2    \bigg[
      \Omega_4 e_1(t)
      -\left(\displaystyle\frac{e_1(t)}{b}+
      \displaystyle\frac{k_3}{b}\right)(e_2(t)+k_3e_3(t))
      -\displaystyle\frac{1}{b k_2}\sin e_3(t)\bigg]
    \\
    \gamma_2 \Omega_1 \bigg[
      -
      \Omega_3
      e_1(t)
      +\left(\displaystyle\frac{e_1(t)}{b}
      +\displaystyle\frac{k_3}{b}\right)(e_2(t)+k_3e_3(t))+
      \displaystyle\frac{1}{b k_2}\sin e_3(t)\bigg]
  \end{pmatrix}
\end{equation*}
where
\begin{alignat*}{3}
  \Delta_r & = \left(1 + \displaystyle
  \frac{\tilde{a}_r(t)}{a_r(t)} \right),
  & \qquad
  \Omega_1 & = \left(v_c(t) + \displaystyle \frac{b}{2} \omega_c(t) \right),
  & \qquad
  \Omega_3 & =
  \left(\displaystyle\frac{e_2(t)}{b}-\displaystyle\frac{1}{2}\right)
  \\
  \Delta_l & = \left(1 + \displaystyle \frac{\tilde{a}_l(t)}{a_l(t)}
  \right),
  &
  \Omega_2 &= \left(v_c(t) - \displaystyle \frac{b}{2} \omega_c(t) \right),
  &
  \Omega_4 &=
  \left(\displaystyle\frac{e_2(t)}{b}+\displaystyle\frac{1}{2}\right)
\end{alignat*}
a vanishing perturbation term $g(t,e_a(t))$ given by
\begin{equation*}
  g(t,e_a(t)) =
  \begin{pmatrix}
    0 \\
    \sigma(t) \left[\displaystyle\frac{1}{2}
      \displaystyle\frac{\tilde{a}_l(t)}{a_l(t)}
      \Omega_2
      + \displaystyle\frac{1}{2}
      \displaystyle\frac{\tilde{a}_r(t)}{a_r(t)}
      \Omega_1
      -k_3 e_3(t) \omega_c(t) + k_1 e_1(t) \right] \\
    0 \\
    0\\
    0
  \end{pmatrix}
\end{equation*}
and a nonvanishing perturbation term $g_n(t,e_a(t))$ given by
\begin{equation*}
  g_n(t,e_a(t)) =
  \begin{pmatrix}
    0 \\
    \sigma(t) v_{\rm{ref}}(t) \cos e_3(t) \\
    0 \\
    -\dot{a}_l(t) \\
    -\dot{a}_r(t)
  \end{pmatrix}
\end{equation*}
with $v_c(t)$ and $\omega_c(t)$ given by
\eqref{eq:auxiliary_control_input_without_slip}, and $e_1(t)$,
$e_2(t)$, and $e_3(t)$ given by \eqref{eq:posture_error}.

Note that the perturbation term $g(t,e_a(t))$ vanishes at the origin $e_a(t)=0$. However, since the parameters $\sigma(t)$, 
$\dot{a}_l(t)$, and $\dot{a}_r(t)$ can be nonzero, the
perturbation term $g_n(t,0)$ in general does not vanish at the
origin. This implies that the origin $e_a(t)=0$ might not be an
equilibrium point of the system
\eqref{eq:nonautonomous_perturbed_system}. Consequently, it is no
longer possible to analyze the stability of the origin $e_a(t)=0$ 
as an equilibrium point of the system
\eqref{eq:nonautonomous_perturbed_system}, nor should it be 
expected that $e_a(t)$ will converge to zero with $t\to \infty$. 
However, it is still possible to study the ultimate boundedness
of the solution $e_a(t)$ of the perturbed system
\eqref{eq:nonautonomous_perturbed_system}. 
For this purpose, and to make the analysis 
tractable, the following set of assumptions will be used:
\begin{Assumption}\label{ass:01}
  The parameters $b$, $k_1$, $k_2$, $k_3$, $\gamma_1$, and $\gamma_2$
  are all finite positive constants. Furthermore:
  \begin{enumerate}
  \item $v_{\rm{ref}}(t)$, $\omega_{\rm{ref}}(t)$,
    $a_l(t)$ and $a_r(t)$ are continuously differentiable,
    and $\sigma(t)$ is piecewise continuous in time.
    
  \item $a_l(t) \ge 1$, $a_r(t) \ge 1$, $v_{\rm{ref}}(t) \ge \mu_1
    >0$, and $|2v_{\rm{ref}}(t) \pm b \omega_{\rm{ref}}(t)| \ge
    \mu_2>0$ for some $\mu_1>0$ and $\mu_2>0$.
	
  \item $v_{\rm{ref}}(t)$, $\omega_{\rm{ref}}(t)$, $a_l(t)$, $a_r(t)$,
    and $\sigma(t)$ are bounded on $[0,\infty)$.

  \item $\dot{v}_{\rm{ref}}(t)$, $\dot{\omega}_{\rm{ref}}(t)$,
  $\dot{a}_l(t)$, $\dot{a}_r(t)$, and $\sigma(t)$ can be made arbitrarily
  small. 
  \end{enumerate} 
\end{Assumption}

The next theorem presents our main result.

\begin{theorem}\label{thm:main_result}
  Under Assumption \ref{ass:01},
  the solution $e_a(t)$ of \eqref{eq:nonautonomous_perturbed_system}
  is uniformly ultimately bounded, for sufficiently small positive
  scalars $\gamma_1$ and $\gamma_2$. Thus,
  there exist positive
  constants $\alpha$, $u_b$, $T$ such that
  \begin{equation*}
    \|e_a(t)\| \leq \alpha \exp\left[-\gamma(t-t_0)\right] \|
    e_a(t_0)\|, \qquad \forall t_0 \leq t<t_0+T
  \end{equation*}
  and
  \begin{equation*}
    \|e_a(t)\| \leq u_b
    , \qquad \forall  
    t\geq t_0+T
  \end{equation*}
\end{theorem}

The proof of Theorem~\ref{thm:main_result} will be pursued in two steps. First,
it is shown in Section~\ref{sec:stabilitynominalsystem} that the origin 
$e_a(t)=0$ of the nominal system
\begin{equation}\label{eq:nominal_system}
  \dot{e}_a=f_a(t,e_a(t))
\end{equation} is exponentially stable 
if the derivatives of $v_{\rm{ref}}(t)$, $\omega_{\rm{ref}}(t)$, $a_l(t)$ and
$a_r(t)$ are sufficiently small, that is, as long as those parameters vary 
slowly. Second, it is shown in Section~\ref{sec:ultimateboundedness} that the
origin $e_a(t)=0$ of the system
\begin{equation}\label{eq:nominal_system_vanishing}
  \dot{e}_a(t)=f_a(t,e_a(t))+g(t,e_a(t))
\end{equation} is also exponentially stable,
which allows us to conclude that the solution $e_a(t)$ of the perturbed
system \eqref{eq:nonautonomous_perturbed_system} is ultimately bounded as
long as the lateral slip $\sigma(t)$ remains small.

\subsection{Exponential stability of the nominal system}\label{sec:stabilitynominalsystem}

It will now be shown that the origin $e_a=0$ of the nonlinear nominal
system \eqref{eq:nominal_system} is exponentially stable using
Theorem~\ref{theo:khalil4133ed}, given in
Appendix~\ref{appx:stabilityperturbed}. This theorem allows us to
assert exponential stability of the origin of
\eqref{eq:nominal_system} based on the exponential stability of the
origin of the linear time-varying
\begin{equation}\label{eq:ltv_system}
  \dot{e}(t) = A(t) e(t)
\end{equation}
with
$A(t)$ the Jacobian of $f_a(t,e_a)$ evaluated at $e_a=0$. The matrix
$A(t)$ is given by
\begin{equation}\label{eq:linearization_matrix}
  A(t)=
  \begin{pmatrix}
    A_{11}(t) & A_{12}(t) \\
    A_{21}(t) & A_{22}(t)
  \end{pmatrix}
\end{equation}
where
\begin{equation*}
  \begin{split}
    A_{11}(t) &=
    \begin{pmatrix}
      -k_1 & \enskip \omega_{\rm{ref}}(t) & \enskip k_3\omega_{\rm{ref}}(t)
      \\
      -\omega_{\rm{ref}}(t) & \enskip 0 & \enskip v_{\rm{ref}}(t)
      \\
      0 & \enskip \displaystyle -\frac{k_2}{2}v_{\rm{ref}}(t)
      & \enskip \displaystyle -\frac{k_4}{2k_3}v_{\rm{ref}}(t)
    \end{pmatrix}
    \\
    A_{12}(t) &=
    \displaystyle\frac{1}{4 b a_l(t) a_r(t)}
    \begin{pmatrix}
      ba_r(t) v_2(t) & -b a_l(t) v_1(t)
      \\
      0 & 0 \\[3pt]
      -2a_r(t) v_2(t) & -2 a_l(t) v_1(t)
    \end{pmatrix}
    \\
    A_{21}(t) &=
    \displaystyle\frac{1}{4b k_2}
    \begin{pmatrix}
      -b k_2 \gamma_1 v_2(t) &
      2k_2k_3 \gamma_1 v_2(t)&
      2 \gamma_1 k_4v_2(t)
      \\
      b k_2\gamma_2v_1(t)&
      2k_2k_3\gamma_2v_1(t)
      &2\gamma_2k_4v_1(t)
    \end{pmatrix}
    \\
    A_{22}(t) &=
    \begin{pmatrix}
      0 & 0
      \\
      0 & 0
    \end{pmatrix}
  \end{split}
\end{equation*}
with $v_1(t) = b \omega_{\rm{ref}}(t) + 2 v_{\rm{ref}}(t)$, $v_2(t) =
b \omega_{\rm{ref}}(t) - 2 v_{\rm{ref}}(t)$, and $k_4 = 1 + k_2
k_3^2$.

The next Remark gathers all conditions on $f_a(t,e_a)$ and its
Jacobian matrix required by Theorems~\ref{theo:khalil4133ed}
and~\ref{theo:rosenbrock63}, given in
Appendix~\ref{appx:stabilityperturbed}.

\begin{remark}\label{rem:boundedness}
  When Assumption~\ref{ass:01} holds, $f_a(t,e_a)$ defined
  in \eqref{eq:nonautonomous_perturbed_system} is continuously
  differentiable, and $f_a(t,0)$ is uniformly bounded for all $t \ge
  0$. Moreover, the Jacobian matrix $J(t,e_a) = \partial f_a(t,e_a)
  / \partial e_a$ is bounded and Lipschitz in $e_a$, uniformly in $t$,
  matrix $A(t)=J(t,0)$, given by \eqref{eq:linearization_matrix}, is
  elementwise bounded and differentiable, and the derivative
  $\dot{A}(t)$ can be made arbitrarily small.
\end{remark}

It now remains to show that the origin of the linear time-varying
system \eqref{eq:ltv_system} is exponentially stable. This can be
accomplished using Theorem~\ref{theo:rosenbrock63}, which asserts the
stability of ``slowly varying'' systems from the stability properties
of the frozen-time systems, i.e., from the eigenvalues of $A(t)$.  The
proof that the eigenvalues of $A(t)$ satisfy the requirement of
Theorem~\ref{theo:rosenbrock63} is the context of the next result.

\begin{proposition}\label{thm:positividade}
  Let $A(t)$ be given by \eqref{eq:linearization_matrix}. Then,
  under Assumption~\ref{ass:01}, there exists a scalar
  $\epsilon > 0$ such that
  \begin{equation*}
    \mbox{Re}[\lambda(A(t))] \leq -\epsilon <0
  \end{equation*}
  for sufficiently small positive scalars $\gamma_1$  and $\gamma_2$.
\end{proposition}

\begin{proof}  
  Consider the characteristic polynomial of matrix $A(t)$, which
  is given by
  \begin{equation}\label{eq:poly_characteristic}
    p(s) = s^5 + \alpha_1(t) s^4 + \alpha_2(t) s^3 + \alpha_3(t) s^2 +
    \alpha_4(t) s + \alpha_5(t),
  \end{equation}
  whose coefficients $\alpha_i(t)$ are given in
  Appendix~\ref{appx:poly}.  
  Under Assumption~\ref{ass:01}, it is clear
  there exists a $\delta_1 > 0$ such that $\alpha_i(t) \ge \delta_1$, 
  for $i=1,\dotsc,5$.

  The second form of the stability criterion of Li\'{e}nard and
  Chipart \citep{Gantmacher:1959:TM}, given in
  Appendix~\ref{appx:LienardChipart}, establishes the necessary
  and sufficient conditions for all the roots of the real
  polynomial $p(s)$, given by \eqref{eq:poly_characteristic}, to
  have negative real parts. These conditions are given by
  \begin{enumerate}
  \item $\alpha_1(t) >0$,  $\alpha_3(t) >0$, and $\alpha_5(t) >0$;
  \item $0 <c_2(t) := \alpha_1(t) \alpha_2(t) - \alpha_3(t)$;
  \item $0 <c_3(t) :=
    \big(\alpha_1(t) \alpha_2(t) - \alpha_3(t)\big) \big(\alpha_3(t)
    \alpha_4(t) - \alpha_2(t) \alpha_5(t)\big) - \big(\alpha_1(t)
    \alpha_4(t) - \alpha_5(t) \big)^2$.
  \end{enumerate}

  Since it is already known that $0 < \delta_1 \le \alpha_i(t)$, for
  $i=1,\dotsc,5$, the first condition is satisfied, and the
  second condition holds true by noticing that the expression
  \begin{equation*}
    \begin{split}
      c_2(t) &=  \alpha_1(t) \alpha_2(t) - \alpha_3(t)
      \\
      & = \frac{1}{16 a_l(t) a_r(t) b^2 k_2 k_3^2}
      \Big\{
      a_l(t) \Big[4 a_r(t) b^2 k_2
        \Big(k_1 (k_4^2 v_{\rm{ref}}^2(t) + 4 k_3^2 \omega_{\rm{ref}}^2(t) )
        \\ & \qquad
        +2 k_1^2 k_3 k_4 v_{\rm{ref}}(t)
        + k_2 k_3 v_{\rm{ref}}(t) (k_4 v_{\rm{ref}}^2(t) + 2 k_3^2
        \omega_{\rm{ref}}^2(t) ) \Big)
        \\
        & \qquad
        +\gamma_2 k_3 v_1^2(t) \left(b^2 k_1 k_2 k_3+2 (k_2^2 k_3^4
        v_{\rm{ref}}(t) + v_{\rm{ref}}(t) ) \right) \Big]
      \\ & \qquad
      +a_r(t) \gamma_1 k_3 v_2^2(t)
      \left( b^2 k_1 k_2 k_3 + 2 (k_2^2 k_3^4
      v_{\rm{ref}}(t) + v_{\rm{ref}}(t) ) \right) \Big\}
      \\
      & 
      > \frac{k_1^2 k_4 v_{\rm{ref}}(t)}{2 k_3}
      \ge v_{\rm{ref}}(t) k_1^2 \sqrt{k_2} \ge
    	  \mu_1 k_1^2 \sqrt{k_2}>0
    \end{split}
  \end{equation*}
  is positive.

  \newcommand\bc{\mathcal{A}}  
  
    It now remains to show the third condition.  For that purpose,
    first notice that the expression for $c_3(t)$ can be rearranged as
    \begin{equation}\label{eq:c3_alternative} 
      c_3(t) = c_2(t) \alpha_3(t) \alpha_4(t) - \rho(t),
      \qquad
      \rho(t) = \alpha_1^2(t) \alpha_4^2(t) + \alpha_5(t) \big(
      c_2(t) \alpha_2(t) - 2 \alpha_1(t) \alpha_4(t) + \alpha_5(t) \big)
    \end{equation}
    and that $\alpha_3(t)$ is bounded below by
    \begin{equation*}
      \alpha_3(t) \ge  v_{\rm{ref}}^2(t) k_2 k_1 /2
    \end{equation*}
    This last inequality directly implies that
    \begin{equation*}
      c_2(t) \alpha_3(t) = \alpha_1(t) \alpha_2(t) \alpha_3(t) -
      \alpha_3^2(t) \ge v_{\rm{ref}}^3(t) k_1^3 \sqrt{k_2^3} /2 \ge 
      \delta_2 >0,
      \quad\text{with $\delta_2 = \mu_1^3 k_1^3 \sqrt{k_2^3}/2$}
    \end{equation*}  
    Notice also that $c_2(t)\alpha_3(t)\alpha_4(t)$ is bounded below
    as follows
    \begin{equation}\label{eq:c2a3a4}
      c_2(t)\alpha_3(t)\alpha_4(t) \ge 
        \frac{k_1 k_3 \mu_1 \mu_2^2 \delta_2}{4 b^2 \bc} \gamma_1
      + \frac{k_1 k_3 \mu_1 \mu_2^2 \delta_2}{4 b^2 \bc} \gamma_2
      + \frac{k_4 \mu_2^4 \delta_2}{16 b^2 k_2 \bc^2}    \gamma_1 \gamma_2 
      > 0
    \end{equation}  
    where $\bc = \sup \left\{a_l(t),a_r(t)\right\}$ is finite
    by Assumption~\ref{ass:01}.
    Likewise, the term $\rho(t)$ in \eqref{eq:c3_alternative} can also
    be written as a polynomial expression in $\gamma_1$ and $\gamma_2$
    as follows
    \begin{equation}\label{eq:rho_poly} 
      \rho(t) =  
      p_{11}(t)\gamma_1\gamma_2	
      + p_{20}(t) \gamma_1^2  
      + p_{02}(t)\gamma_2^2  + 	p_{21}(t) \gamma_1^2 \gamma_2  
      + p_{12}(t) \gamma_1 \gamma_2^2 + \dotsb
    \end{equation}
    where the coefficients $p_{ij}(t)$ are bounded by
    Assumption~\ref{ass:01}. Thus, for sufficiently small $\gamma_1$
    and $\gamma_2$, the first two terms on the right-hand side of
    \eqref{eq:c2a3a4} dominate all the remaining terms in the
    expression for $c_3(t)$ given by \eqref{eq:c3_alternative}, and
    hence there is a $\delta_3>0$ such that $c_3(t) \ge \delta_3$ for
    all $t\ge 0$.

  This concludes our proof and it has been shown (see
  Theorem~\ref{theo:lienard} and Remark~\ref{rem:strictly} in
  Appendix~\ref{appx:LienardChipart}) that $\mbox{Re}[\lambda_j(A(t))]
  \leq -\epsilon <0$.
  
\end{proof}

Therefore, from Proposition~\ref{thm:positividade}, all conditions of
Theorem~\ref{theo:rosenbrock63} have been satisfied, which means there
is some $\epsilon_d > 0$ such that if $|\dot{a}_{ij}|\leq \epsilon_d$,
the origin $e_a(t)=0$ of the linear time-varying system
\eqref{eq:ltv_system} is exponentially stable. From
\eqref{eq:linearization_matrix}, it is readily seen that, for
arbitrarily small $\epsilon_d > 0$, if $\dot{v}_{\rm{ref}}(t)$,
$\dot{\omega}_{\rm{ref}}(t)$, $\dot{a}_r(t)$, and $\dot{a}_l(t)$ are
bounded by sufficiently small values, it is possible to enforce
$|\dot{a}_{ij}(t)|\leq \epsilon_d$. In this case, the origin $e_a=0$
of the linear time-varying system \eqref{eq:ltv_system} is
exponentially stable and, from Theorem~\ref{theo:khalil4133ed}, the
origin $e_a(t)=0$ of the nonlinear system \eqref{eq:nominal_system} is
also exponentially stable.

\subsection{Ultimate boundedness of solutions of the perturbed
  system}\label{sec:ultimateboundedness}

Since the origin $e_a(t)=0$ of the nominal system \eqref{eq:nominal_system}
has been shown to be exponentially stable, under
Assumption \ref{ass:01}, Lemma~\ref{lem:khalil91} (given in
Appendix~\ref{appx:stabilityperturbed}) is applied to show that the
origin $e_a(t)=0$ of the perturbed system
\eqref{eq:nominal_system_vanishing} is also exponentially stable.
By exponential stability of the origin, the existence of a Lyapunov
function for the nominal system \eqref{eq:nominal_system} that
satisfies conditions
\eqref{appx:eq:contraint_01}-\eqref{appx:eq:contraint_03} is ensured by a 
Lyapunov converse theorem (see, for instance, Theorem~4.14 from 
\cite{Khalil:2002:NS}), provided that $\partial f/\partial e_a$ is bounded 
on a domain $D = \{e_a \in \mathbb{R}^5 \;|\; \|e_a(t)\|<d \}$,
uniformly in $t$. This is indeed the case, since $f_a(t,e_a)$ is
composed of polynomial and sinusoidal functions of $e_a$ and
continuous time-varying functions that are bounded for all $t \ge 0$.
To successfully conclude the application of Lemma~\ref{lem:khalil91},
it remains to show that conditions
\eqref{appx:eq:contraint_04_vanishing} and
\eqref{appx:eq:contraint_05_vanishing} are satisfied.  To verify this,
first write $g(t,e_a) = \sigma(t) \bar{g}(t,e_a)$. The vector function
$\bar{g}(t,e_a)$ is continuously differentiable in $e_a$, and thus
$\|\partial \bar{g}/\partial e_a\|\le \bar{b}$ on $[0,\infty)\times D$
for some $\bar{b}>0$ (the bound $b$ is independent of $t$, due to
Assumption \ref{ass:01}). Thus $\bar{g}(t,e_a)$ is Lipschitz in $e_a$
on the compact domain $D$, uniformly in $t$, with Lipschitz constant
$\bar{b}$ (see, for instance, \cite[Lemma 3.1]{Khalil:2002:NS}). Since
$\bar{g}(t,0)=0$ and $\bar{g}(t,e_a)$ is Lipschitz in $e_a$, it
follows that $\|g(t,e_a\| \le |\sigma(t)| \bar{b} \|e_a\|$, and
condition \eqref{appx:eq:contraint_04_vanishing} is verified; this
also implies that, for a sufficiently small bound on $|\sigma(t)|$,
condition \eqref{appx:eq:contraint_05_vanishing} is verified. Thus,
the origin of the nominal system perturbed by the vanishing term,
given by \eqref{eq:nominal_system_vanishing}, is also exponentially
stable, provided that the magnitude of the lateral slip is bounded by
a sufficiently small value.

Since the origin $e_a(t)=0$ is an exponentially stable equilibrium point
of the system \eqref{eq:nominal_system_vanishing}, 
Lemma~\ref{lem:khalil92} (given in Appendix~\ref{appx:stabilityperturbed}) 
can be applied to show
that the solution $e_a(t)$ of the perturbed system 
\eqref{eq:nonautonomous_perturbed_system} is ultimately bounded by a
small bound. For this purpose, the perturbed system
\eqref{eq:nonautonomous_perturbed_system} is rewritten as
\begin{equation*}
  \dot{e}_a = f(t,e_a) + g_n(t,e_a)
\end{equation*}
with $f(t,e_a) = f_a(t,e_a) + g(t,e_a)$.
To use Lemma~\ref{lem:khalil92}, it is necessary to find a Lyapunov
function $V(t,e_a(t))$ for the system
\eqref{eq:nominal_system_vanishing} that satisfies the conditions
\eqref{appx:eq:contraint_01}, \eqref{appx:eq:contraint_02}, and
\eqref{appx:eq:contraint_03}. This Lyapunov function indeed exists,
since the Lyapunov function for the system \eqref{eq:nominal_system}
is also a Lyapunov function for the system
\eqref{eq:nominal_system_vanishing}. This fact is shown in the proof
of Lemma~\ref{lem:khalil91} in \cite{Khalil:2002:NS}. To conclude the
application of Lemma~\ref{lem:khalil92}, it remains to show that
condition \eqref{appx:eq:contraint_04_nonvanishing} is satisfied. 
The proof of this fact uses exactly the same reasoning used to
prove that conditions \eqref{appx:eq:contraint_04_vanishing} and
\eqref{appx:eq:contraint_05_vanishing} are satisfied.  
This concludes the proof of Theorem \ref{thm:main_result}.

\section{Numerical results}\label{sec:numerical_results}
This section presents some numerical simulations of the proposed
adaptive control strategy given in
Section~\ref{sec:adaptive_control_strategy}. These simulations aim to
provide an illustration of our theoretical results, and give an idea
of the performance of the adaptive kinematic controller (AKC) whose
block diagram is shown in
Figure~\ref{fig:adaptive_kin}.
To illustrate how adaptation affects performance, the non-adaptive
kinematic controller (NKC) given
by \eqref{eq:auxiliary_control_input_without_slip} will also be
simulated, as it can be viewed as a version of the adaptive scheme
where adaptation is ``turned off''.

In all simulations, the physical
parameters of the wheeled robot model, taken from 
\cite{Ryu:2011:DFB}, are given by $b=0.1624$ m and
$r=0.0825$ m. The reference trajectory used in the simulations is also
fixed, and is described in Appendix~\ref{appx:reference}. Fixing the
reference trajectory in this numerical analysis is justified by the
fact that, in many applications, the robot must follow a specific
planned trajectory. Figure~\ref{fig:ref_input_prof01} shows the
reference trajectory $(x_{\rm{ref}}(t),y_{\rm{ref}}(t))$ generated
using the input $\eta_{\rm{ref}}(t)={(v_{\rm{ref}}(t),
\omega_{\rm{ref}}(t))}^T$. Note that the reference trajectory
can be divided in three parts. The first part consists of a
straight line path that occurs on the interval 
$0 \ \mbox{s} \le t < 25 \ \mbox{s}$, the second part consists
of a curved path with time-varying radius of curvature that
occurs on the interval $25 \ \mbox{s} \le t < 45 \ \mbox{s}$,
and the third part consists of a straight line path that
occurs on the interval $t \geq 45 \ \mbox{s}$.

\begin{figure}[ht]
  \centering
  \subfigure{%
    \begin{psfrags}
      \psfrag{t}[tc][c]{$t$ (s)}
      \psfrag{ref}[c][c]{$v_{\rm{ref}}$ (m/s) and $\omega_{\rm{ref}}$ (rad/s)}
      \psfrag{vref}[l][l]{$v_{\rm{ref}}$}
      \psfrag{wref}[l][l]{$\omega_{\rm{ref}}$}
      \includegraphics[width=0.5\linewidth]{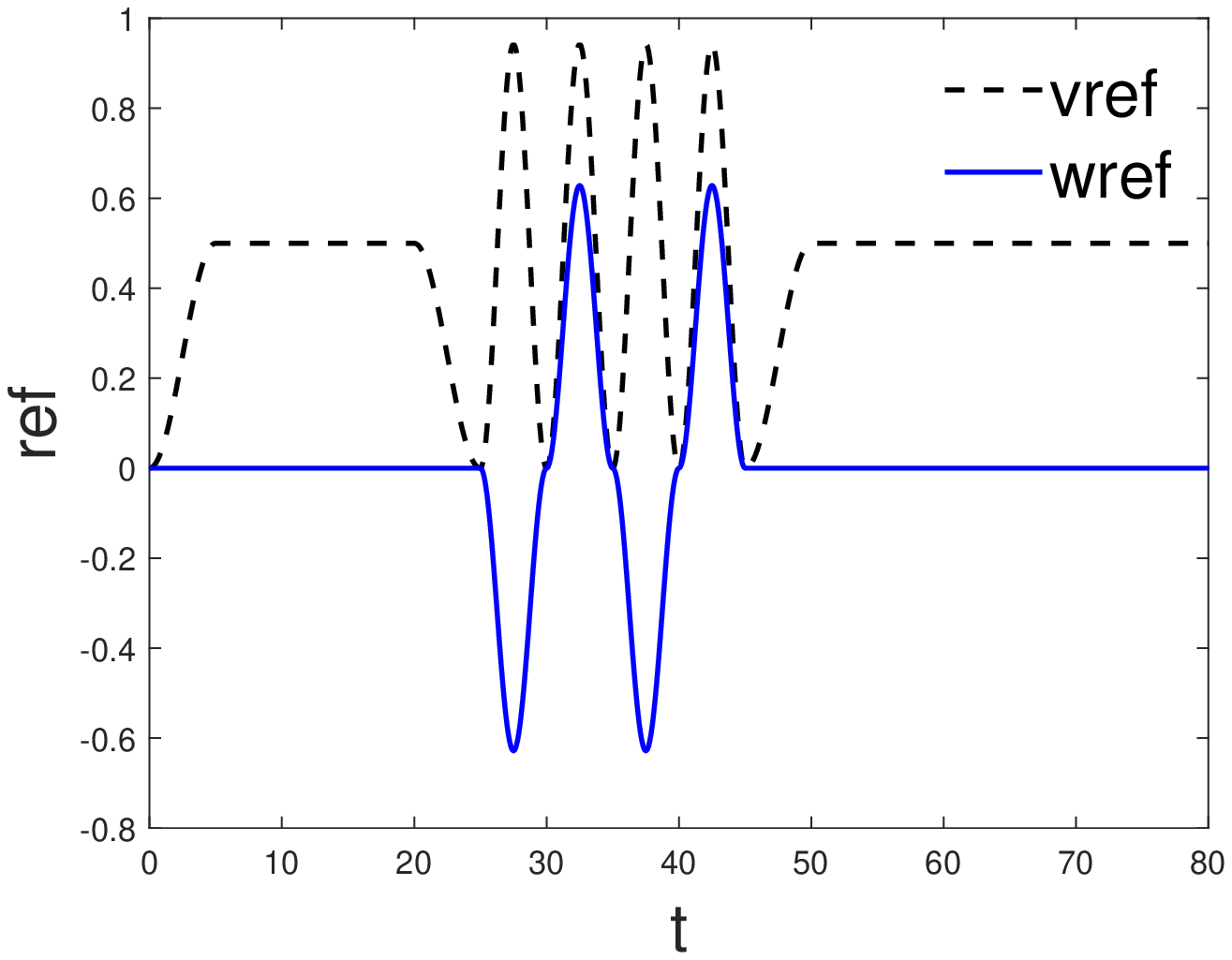}
    \end{psfrags}
  }%
  \subfigure{%
    \begin{psfrags}
      \psfrag{X}[tc][c]{$x_{\rm{ref}}(t)$ (m)}
      \psfrag{Y}[c][c]{$y_{\rm{ref}}(t)$ (m)}
      \includegraphics[width=0.5\linewidth]{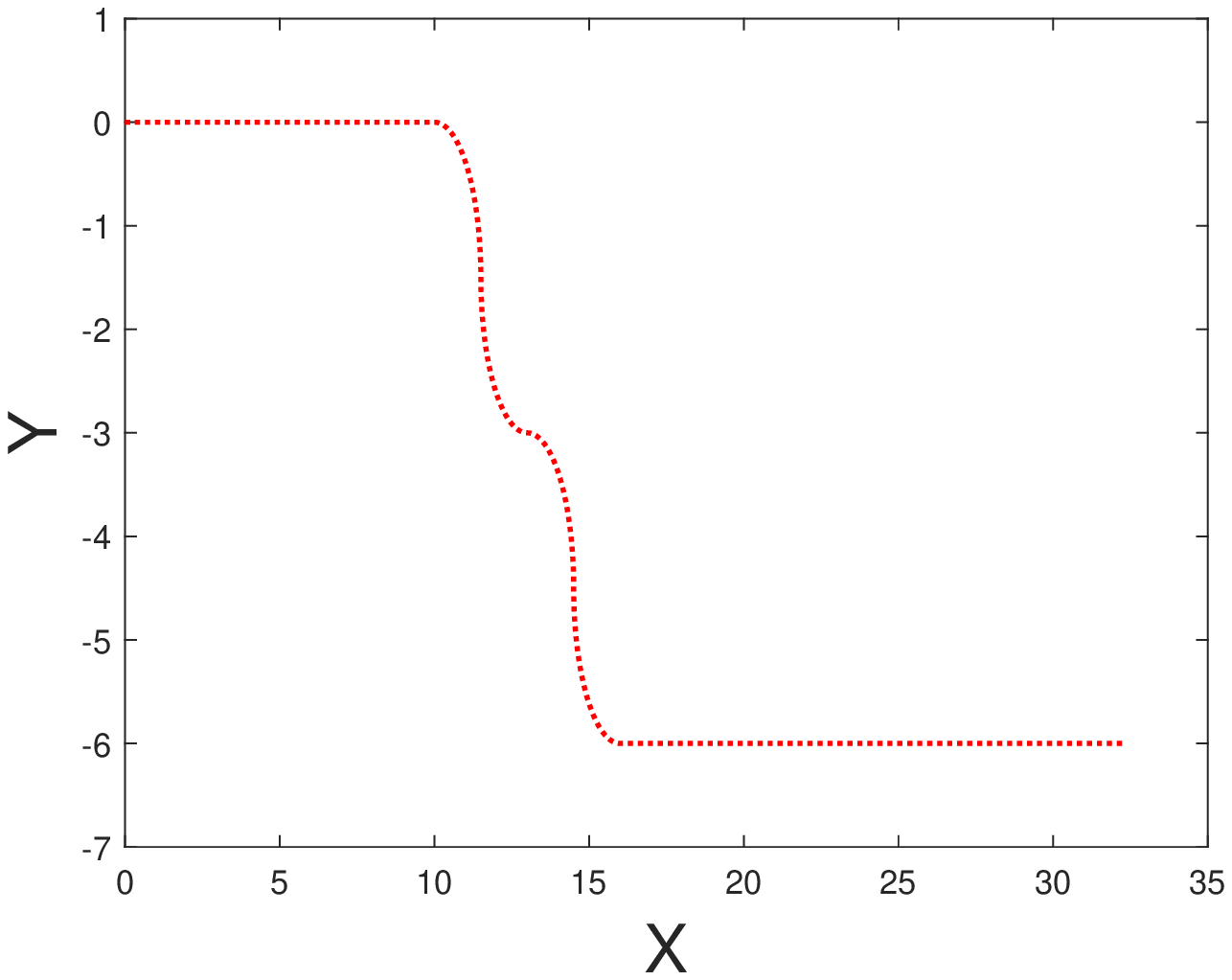}
    \end{psfrags}
  }
  \caption{Reference trajectory 
  $(x_{\rm{ref}}(t),y_{\rm{ref}}(t))$
    generated by the reference input
    $\eta_{\rm{ref}}(t)={(v_{\rm{ref}}(t),\omega_{\rm{ref}}(t))}^T$.\label{fig:ref_input_prof01}}
\end{figure}

To simulate the controlled robot, it is also necessary to select the
controller gains.
A standard criterion for choosing optimal controller
gains involves minimizing the cost function given by
\begin{equation*}
  \mathcal{F} = \int_0^{t_f}  e^T Q e  + 
  (\xi-\xi_{\rm{ref}})^T R (\xi-\xi_{\rm{ref}}) \,\rm{d}t
\end{equation*} 
with $Q>0$ and $R\ge 0$, which represents  weights energy of
states and control actions spent while the robot is moving. This
criterion is adopted to select gains for both the NKC and AKC schemes,
using $Q=I$ and $R = 0.05I$. Notice that the term $\xi-\xi_{\rm{ref}}$
represents the difference between the controlled wheel speeds and the
reference wheel speeds resulting from the reference trajectory
$\eta_{\rm{ref}}$. This difference depends on the longitudinal slip
parameters $a_l(t)$ and $a_r(t)$, which also need to be provided so that
$\mathcal{F}$ can be minimized. Since the NKC is not designed to take
into account the slip, for the purposes of gain selection, the robot
controlled by the NKC is subjected to zero slip. Thus, for the NKC,
$\xi-\xi_{\rm{ref}} =
\Phi_0^{-1} (\eta_c(t)-\eta_{\rm{ref}})$, where
$\Phi_0$ denotes the matrix \eqref{eq:slip_matrix} with $a_l = a_r =
1$. The AKC, on the other hand, has to be able to deal with the slip.
For this reason, to select the gains for the AKC, a ``training'' slip
profile given by $a_l(t) = 5/3 (H(t-15)-H(t-50)$ and $a_r(t) = 5/2
(H(t-30)-H(t-65))$ is used , where $H(\cdot)$ denotes the Heaviside
step function. Using this slip, the controlled and reference wheel
speeds for the AKC are given by $\xi = \hat{\Phi}(t)^{-1}\eta_c(t)$
and $\xi_{\rm{ref}} = \Phi(t)^{-1}\eta_{\rm{ref}}$, respectively.

To select controller gains for the NKC and AKC schemes, a grid search
over the gains $k_i$ ($i=1,2,3$) was performed, and those gains that
yielded the smallest $\mathcal{F}$ were selected. The grid was chosen
such that $k_1$, $k_2$ and $k_3$ range over $20$ equally spaced points
in logarithmic scale between $10^{-1}$ and $10$. For the AKC scheme,
the values $\gamma_1=\gamma_2=3$ were fixed a priori. The minimum
value of $\mathcal{F}$ using the AKC scheme is indicated in
Figure~\ref{fig:calculate_parameters_prof01_akc} by a circle, and was
obtained when $(k_1,k_2,k_3)=(1.44,10,1.83)$. Using the NKC scheme,
the gains $(k_1,k_2,k_3)=(0.26,10,0.1)$ provided the smallest value
for $\mathcal{F}$, as indicated in
Figure~\ref{fig:calculate_parameters_prof01_nac}. The set of gains
that provided a minimum $\mathcal{F}$ using each controller will be
called henceforth the best-known gains.

\begin{figure}[ht]
  \centering
  \subfigure[AKC design ($k_2=10$)]{%
    \begin{psfrags}
      \psfrag{k1}[c][c]{$k_1$}
      \psfrag{k3}[bc][tc]{$k_3$}
      \psfrag{F}[c][l]{\scriptsize $\log_{10}{(\mathcal{F})}$}
      \includegraphics[width=0.5\linewidth]{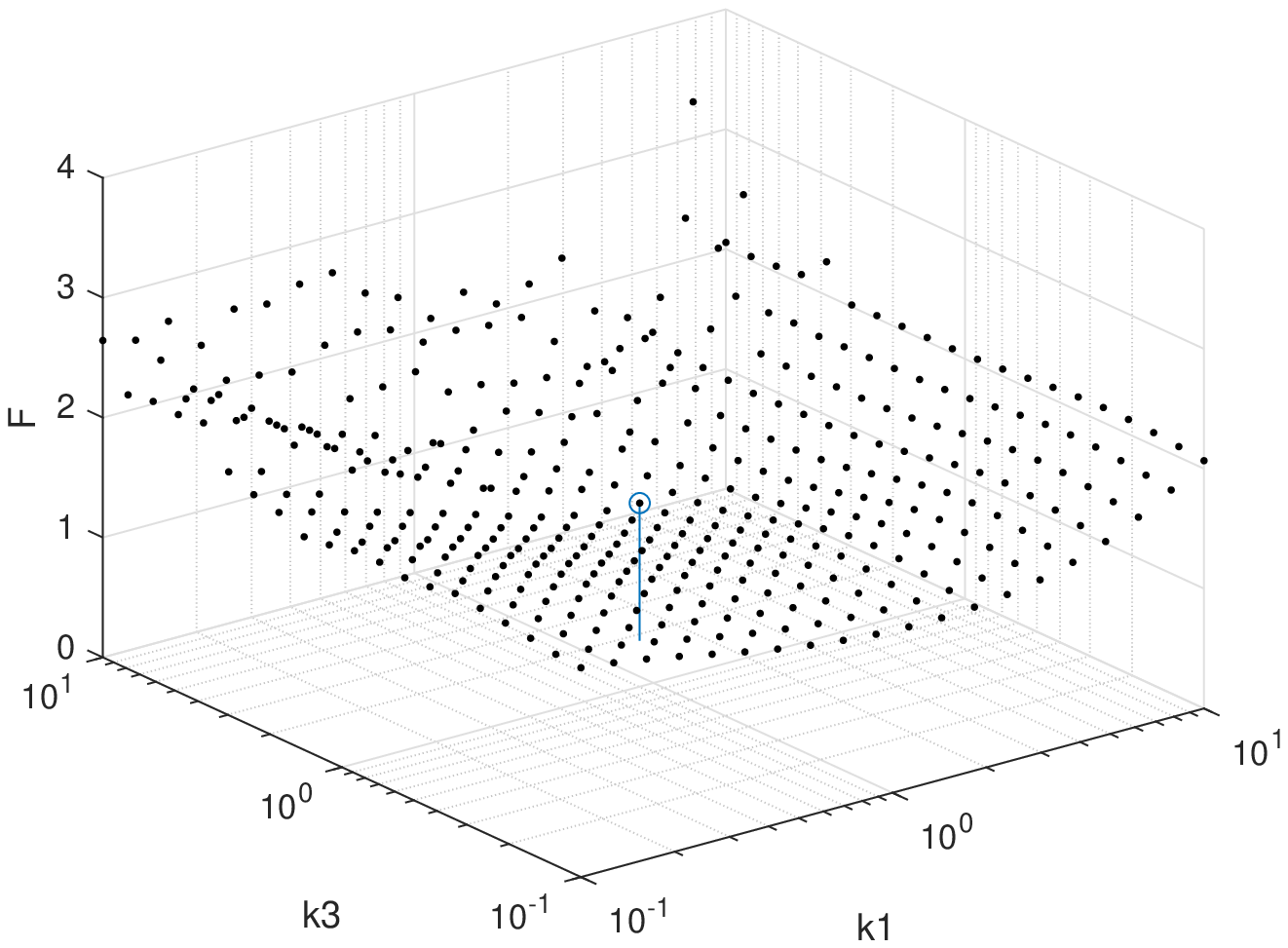}
    \end{psfrags}
    \label{fig:calculate_parameters_prof01_akc}
  }%
  \subfigure[NKC design ($k_2=10$)]{%
    \begin{psfrags}
      \psfrag{k1}[c][c]{$k_1$}
      \psfrag{k3}[bc][tc]{$k_3$}
      \psfrag{F}[c][l]{\scriptsize $\log_{10}{(\mathcal{F})}$}
      \includegraphics[width=0.5\linewidth]{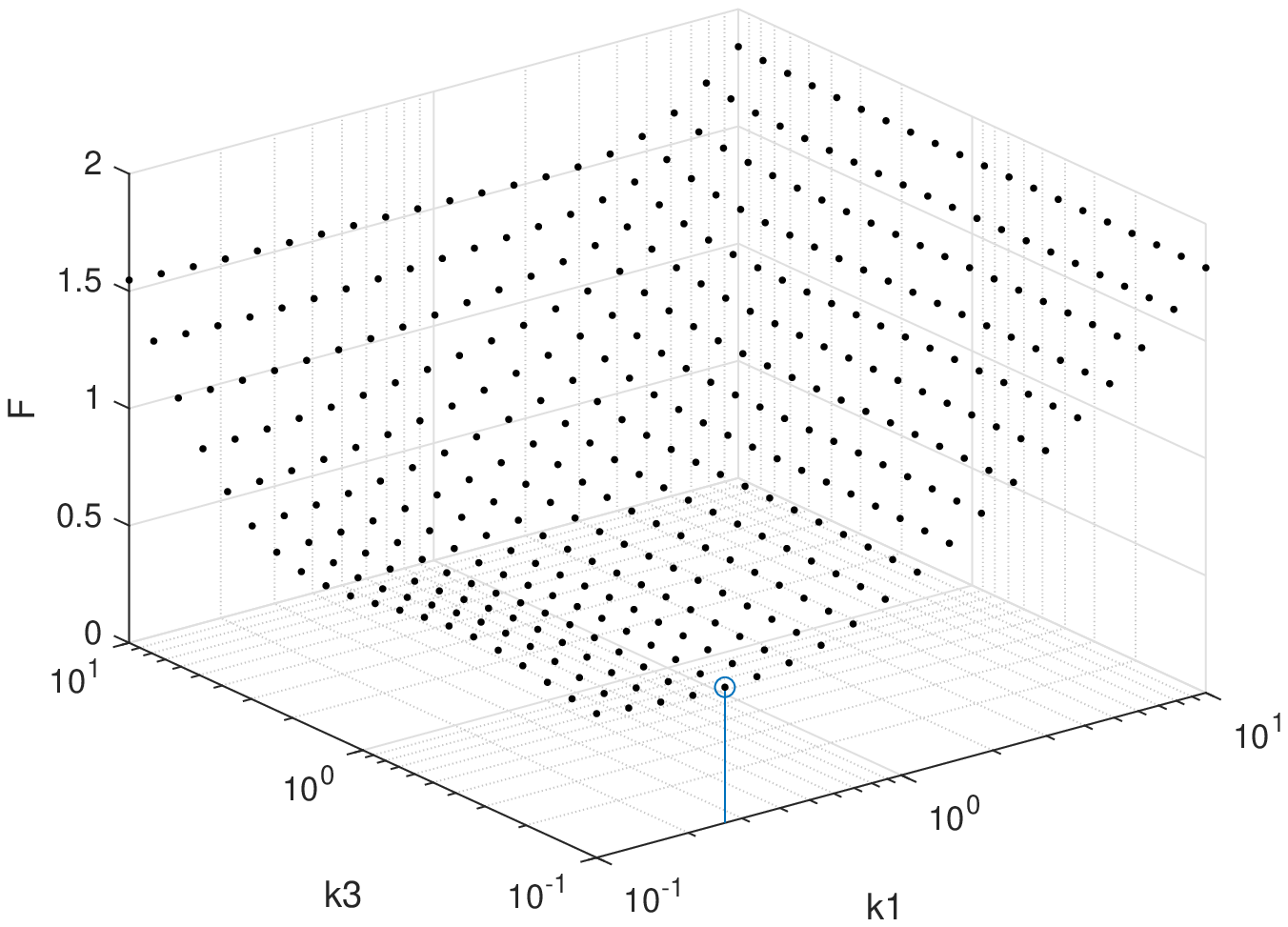}
    \end{psfrags}
    \label{fig:calculate_parameters_prof01_nac}
  }
  \caption{Value of $\mathcal{F}$ for $k_1$ and $k_3$ for the fixed value $k_2=10$.}
  \label{fig:calculate_parameters_prof01}
\end{figure}

Once the controller gains have been selected, to verify the performance of
the NKC and AKC schemes, the robot under both schemes is subjected to
a ``validation'' slip profile, which is a more challenging slip
profile, chosen so as to vary significantly in time. The longitudinal
slip parameters $a_l(t)$ and $a_r(t)$ and the lateral slip parameter
$\sigma(t)$ are chosen as the following nonlinear time-varying signals
\begin{equation*}
  \begin{split}
    a_l(t) &= 1/\left(0.7+ 0.3 \exp(-0.1 t) \cos(1.1t)\right) \\
    a_r(t) &= 1/\left(0.4+0.6 \exp(-0.08 t) \sin^2(0.02 t^2)
    \right) \\
    \sigma(t) &=3.0\exp(-0.03t)\text{sinc}(t-35.0)
  \end{split}
\end{equation*}

Figure~\ref{fig:slip_parameters_prof01} shows the longitudinal slip
parameters $a_l(t)$ and $a_r(t)$ and the lateral slip parameter
$\sigma(t)$. Physically, it is expected that a high variation of the
longitudinal slip will occur when the robot is accelerating. This
occurs, for example, when the robot starts its motion from rest. It is
also expected that a high variation of the lateral slip will occur
when the robot moves along a curve. Note that a high value of the
lateral slip occurs in the middle of the interval $25 \ \mbox{s} \le t
< 45 \ \mbox{s}$, where the reference trajectory is a curved path. It
is well known that the lateral slip is zero during straight line
motion. However, different values of longitudinal slip may cause small
lateral slip.

\begin{figure}[ht]
  \centering
  \subfigure{%
    \begin{psfrags}
      \psfrag{t}[tc][c]{$t$ (s)}
      \psfrag{slipf}[c][c]{Longitudinal slip factors}
      \psfrag{aL}[l][l]{$a_l$}
      \psfrag{aR}[l][l]{$a_r$}
      \includegraphics[width=0.5\linewidth]{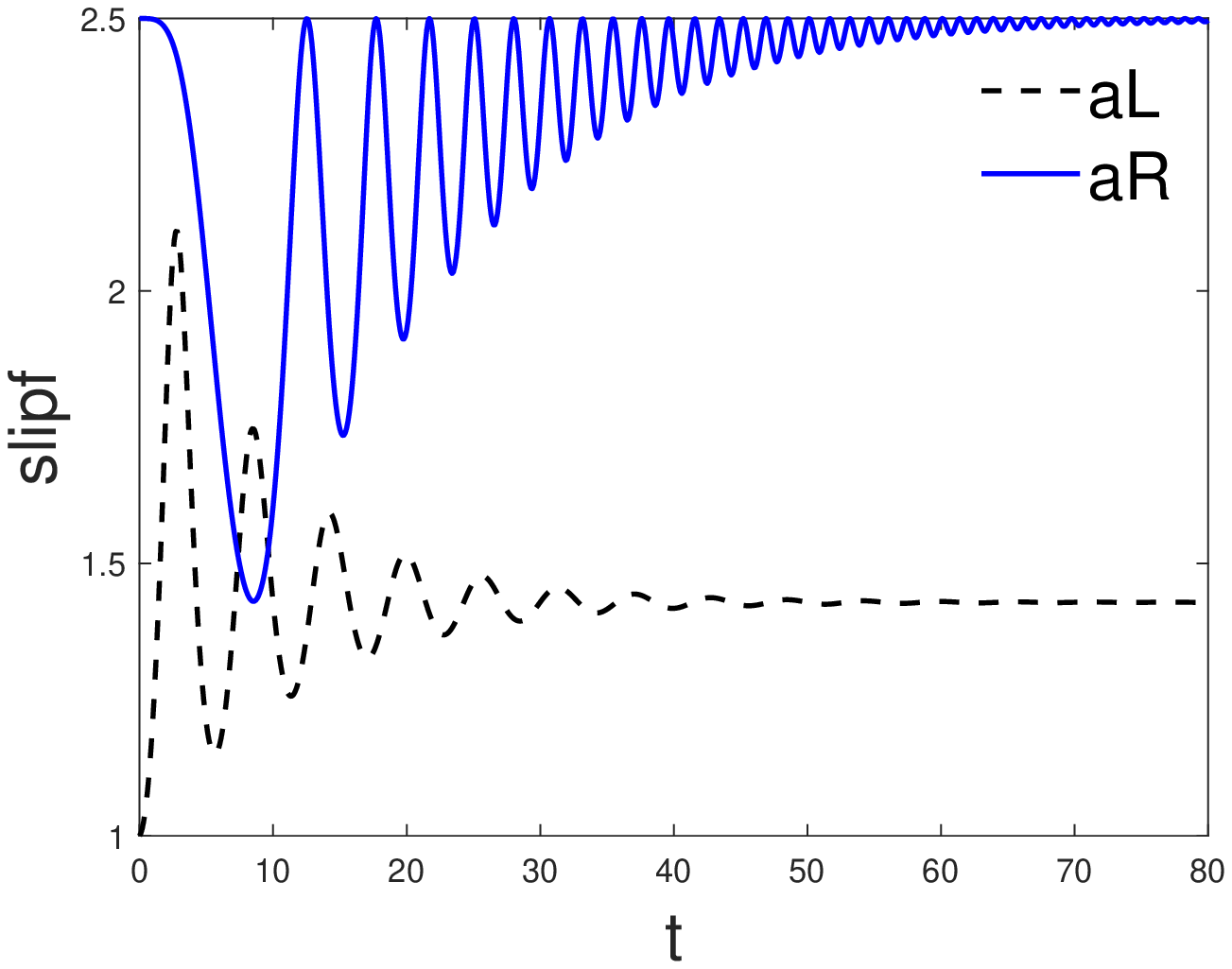}
    \end{psfrags}
  }%
  \subfigure{%
    \begin{psfrags}
      \psfrag{t}[tc][c]{$t$ (s)}
      \psfrag{slipp}[c][c]{Lateral slip parameter}
      \psfrag{sig}[l][l]{$\sigma$}
      \includegraphics[width=0.5\linewidth]{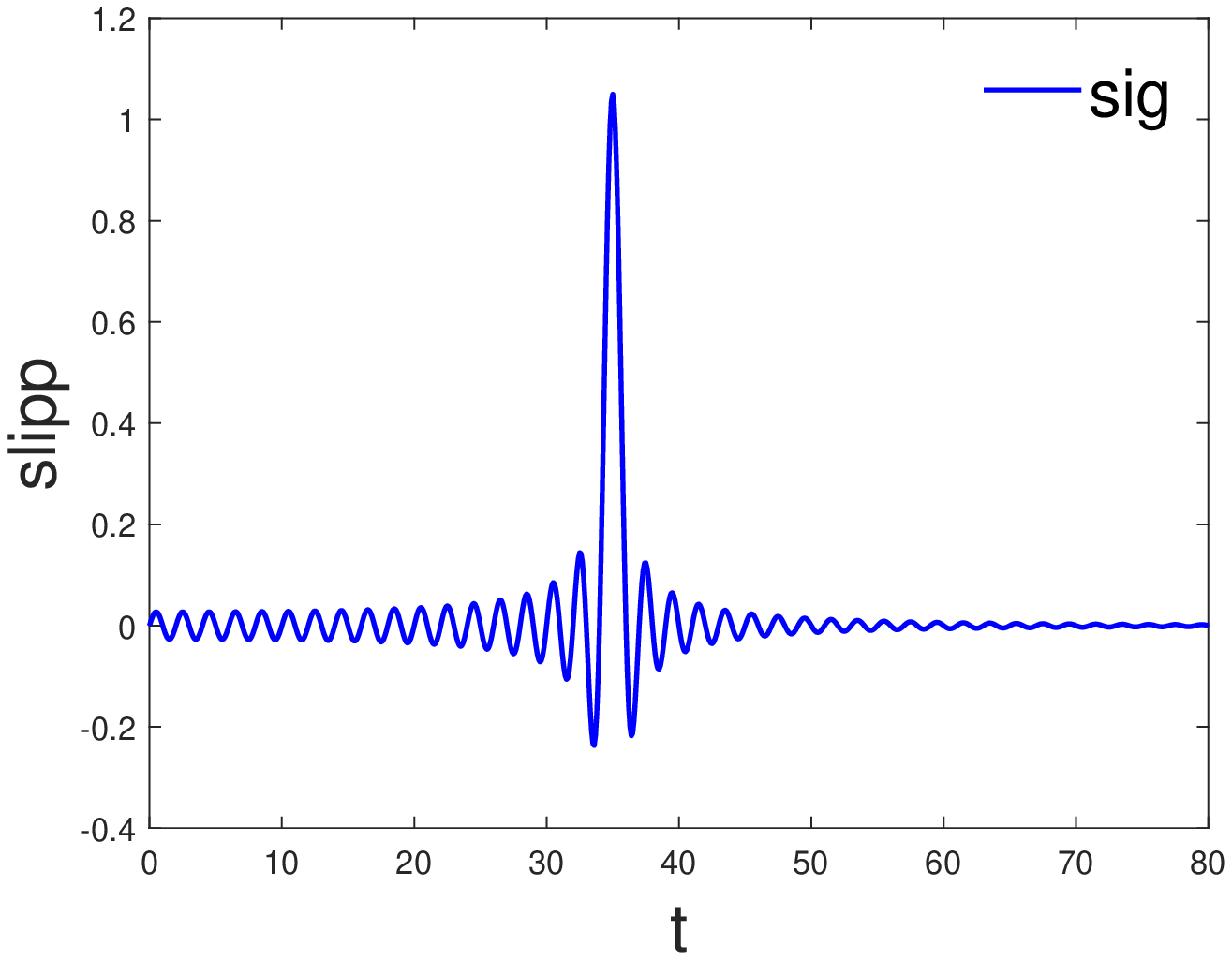}
    \end{psfrags}
  }
  \caption{Longitudinal slip parameters $a_l(t)$ and $a_r(t)$ and
    lateral slip parameter $\sigma(t)$.
    \label{fig:slip_parameters_prof01}}
\end{figure}

Figures~\ref{fig:tracking_ref_trajectory_prof01}
through~\ref{fig:effective_inputs_prof01} show the simulations using
the AKC and NKC schemes with the best-known gains.
Figure~\ref{fig:tracking_ref_trajectory_prof01} shows the robot
trajectory obtained using the NKC and AKC schemes. The dotted line,
the dashed line and the solid line stand, respectively, for the
reference trajectory, the robot trajectory obtained using the NKC
scheme, and the robot trajectory obtained using the AKC scheme. The
initial condition of the robot, indicated by a black circle, is
$q_{\rm{p}}(0)={(1/2,-3/4,-\pi/6)}^T$. Note that the robot with the
AKC scheme is able to follow the reference trajectory with a small
error. On the other hand, the NKC scheme is not able to compensate for
the slip, and consequently, the robot trajectory diverges with respect
to the reference trajectory.

\begin{figure}[ht]
  \parbox[t]{0.5\linewidth}{%
    \centering
    \begin{psfrags}
      \psfrag{X}[tc][c]{$x$ (m)}
      \psfrag{Y}[c][c]{$y$ (m)}
      \psfrag{TR}[l][l]{\footnotesize RT}
      \psfrag{CNC}[l][l]{\footnotesize NKC}
      \psfrag{CAC}[l][l]{\footnotesize AKC}
      \includegraphics[width=\linewidth]{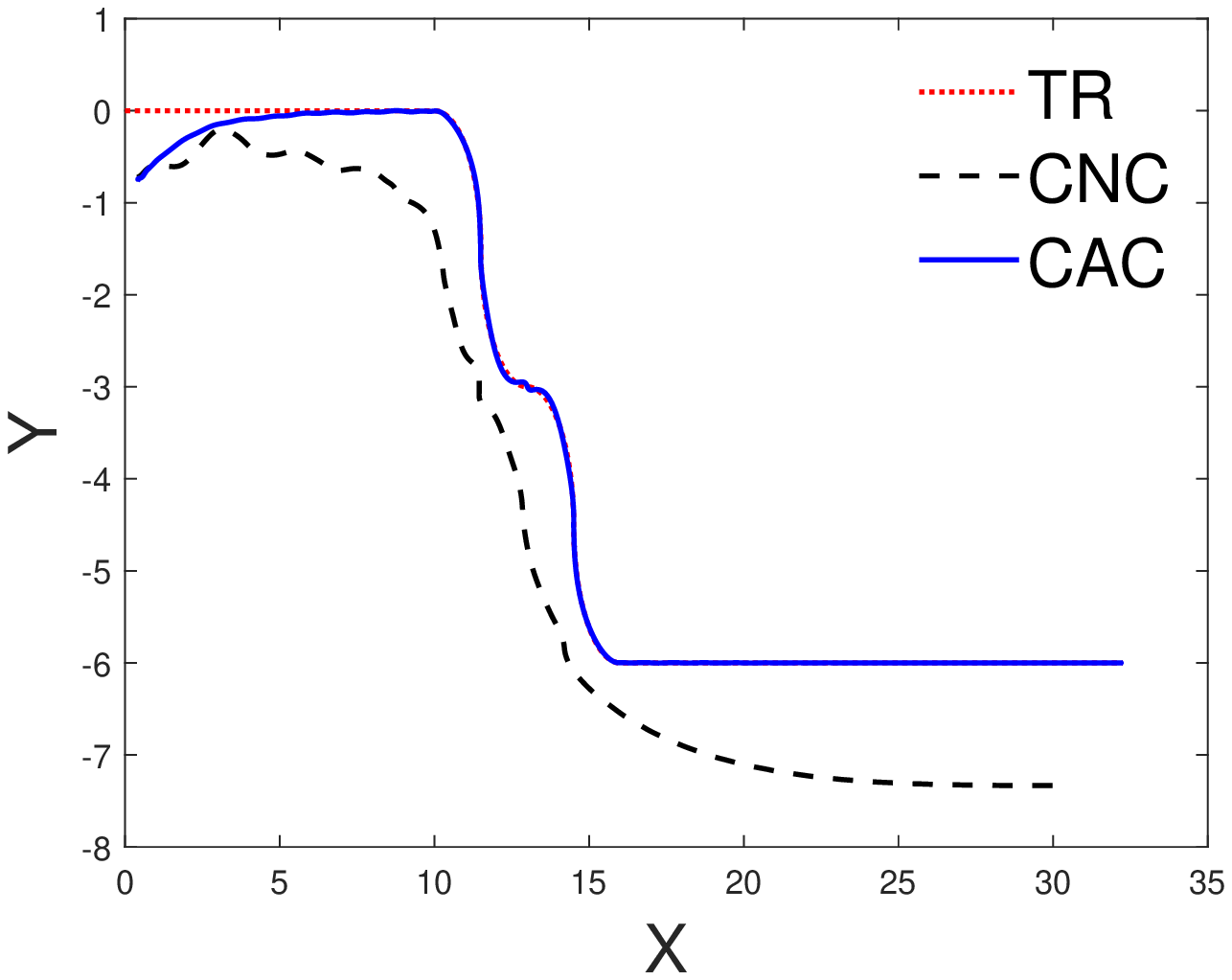}
    \end{psfrags}
    \caption{Reference trajectory (RT) and robot trajectory using
    the NKC and AKC schemes.}    
    \label{fig:tracking_ref_trajectory_prof01}
  }%
  \parbox[t]{0.5\linewidth}{%
    \centering
    \begin{psfrags}
      \psfrag{t}[tc][c]{$t$ (s)}
      \psfrag{slip}[c][c]{Estimation errors}
      \psfrag{aLt}[l][l]{\footnotesize $\tilde{a}_l$}
      \psfrag{aRt}[l][l]{\footnotesize $\tilde{a}_r$}
      \includegraphics[width=\linewidth]{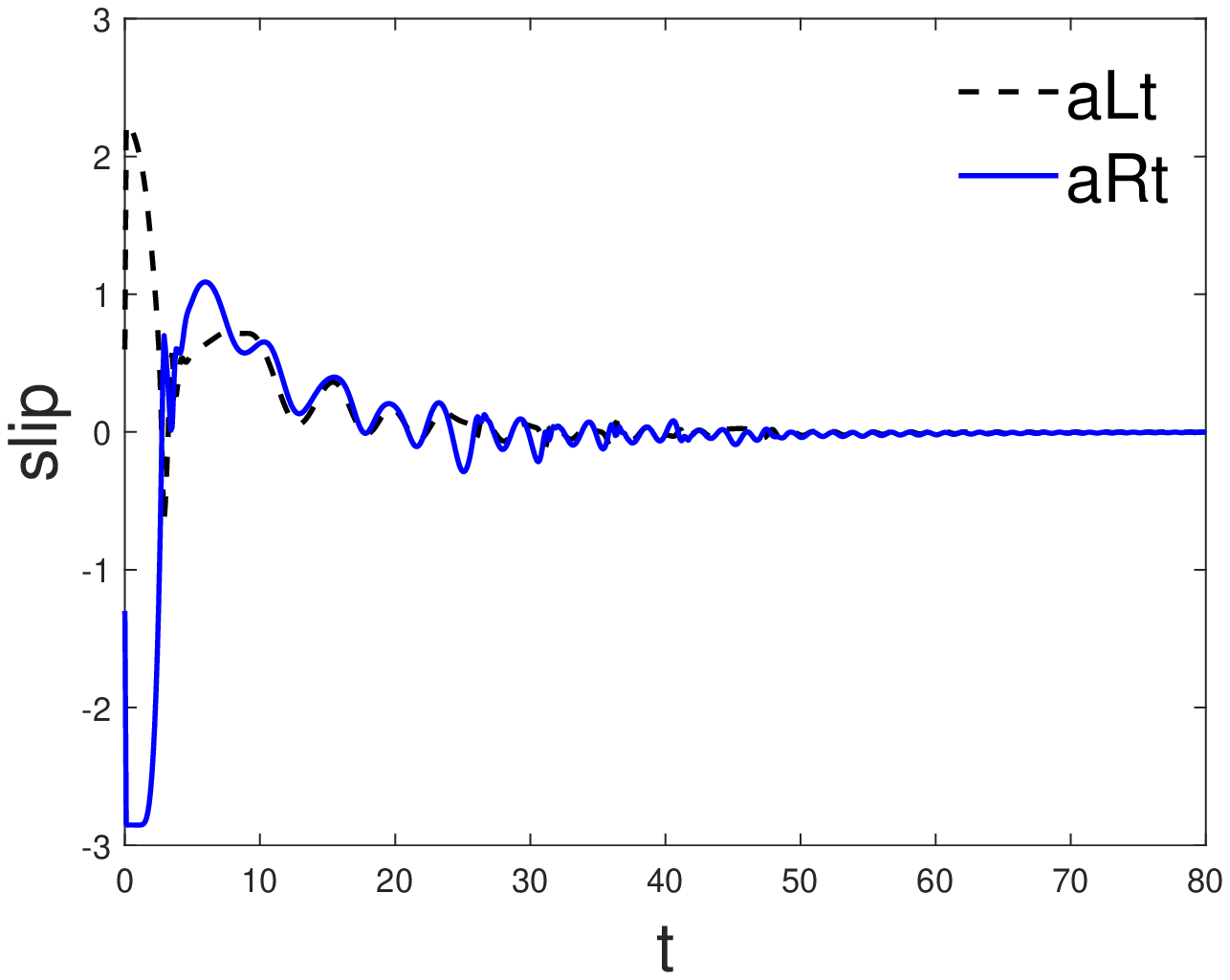}
    \end{psfrags}
    \caption{Estimation errors $\tilde{a}_l$ and
    $\tilde{a}_r$.}    
    \label{fig:estimation_error_aLR_prof01}
  }
\end{figure}

Figure~\ref{fig:estimation_error_aLR_prof01} shows the estimation
errors $\tilde{a}_l$ (dashed) and $\tilde{a}_r$ (solid). The initial
conditions of the update rule were taken as $\hat{a}_l(0)=1.6$ and
$\hat{a}_r(0)=1.2$, which differ from the true values $a_l(0)=1$ e
$a_r(0)=2.5$. Note that an estimation error occurs at the beginning of
the trajectory. However, this is not surprising, since convergence is
only guaranteed for constant slip parameters.

\begin{figure}[ht]
  \parbox{0.5\linewidth}{%
    \centering
    \begin{psfrags}
      \psfrag{t}[tc][c]{$t$ (s)}
      \psfrag{E1}[c][c]{$e_1$ (m)}
      \psfrag{CNC}[l][l]{\footnotesize NKC}
      \psfrag{CAC}[l][l]{\footnotesize AKC}
      \includegraphics[width=\linewidth]{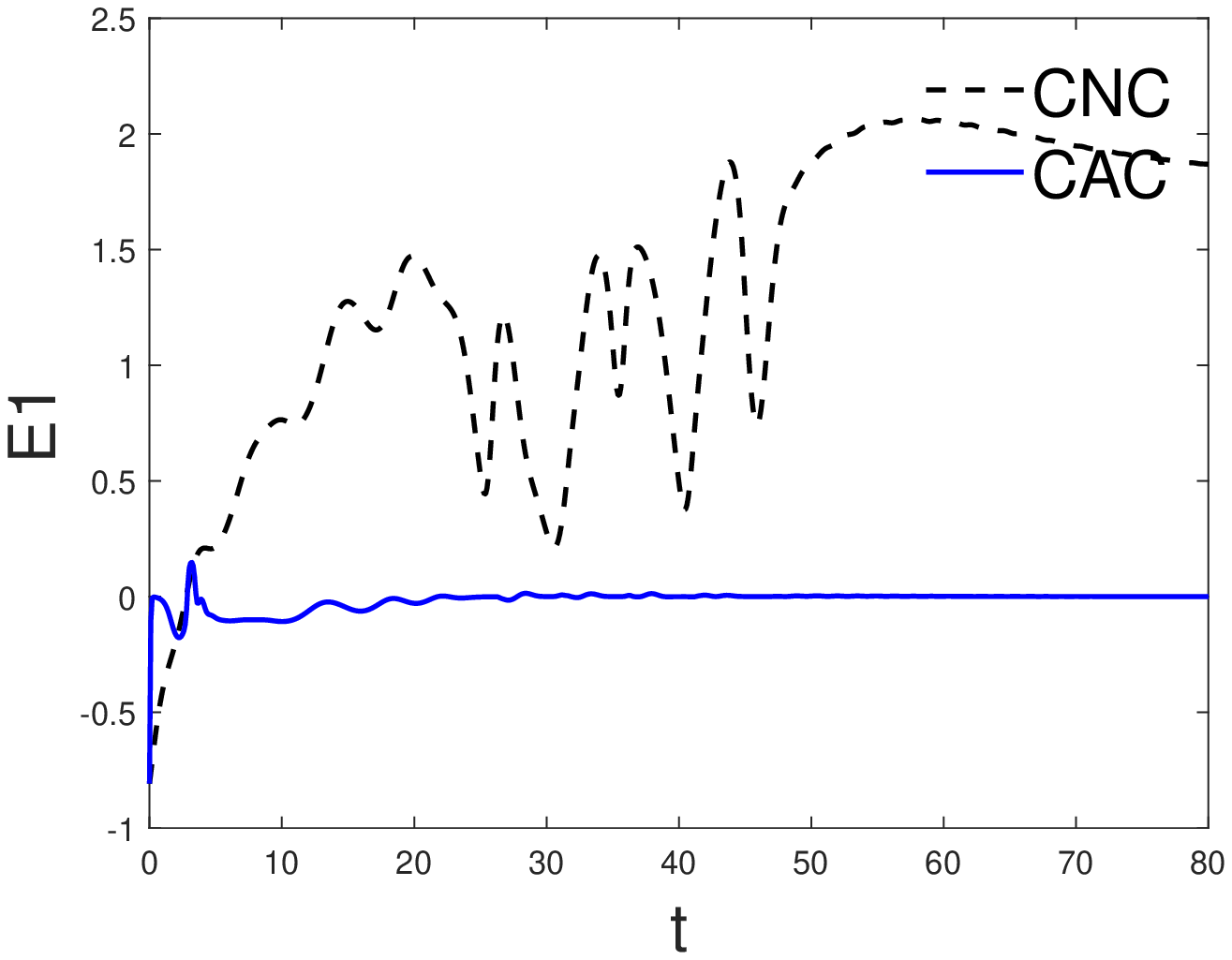}
    \end{psfrags}
  }%
  \parbox{0.5\linewidth}{%
    \centering  
    \begin{psfrags}
      \psfrag{t}[tc][c]{$t$ (s)}
      \psfrag{E2}[c][c]{$e_2$ (m)}
      \psfrag{CNC}[l][l]{\footnotesize NKC}
      \psfrag{CAC}[l][l]{\footnotesize AKC}
      \includegraphics[width=\linewidth]{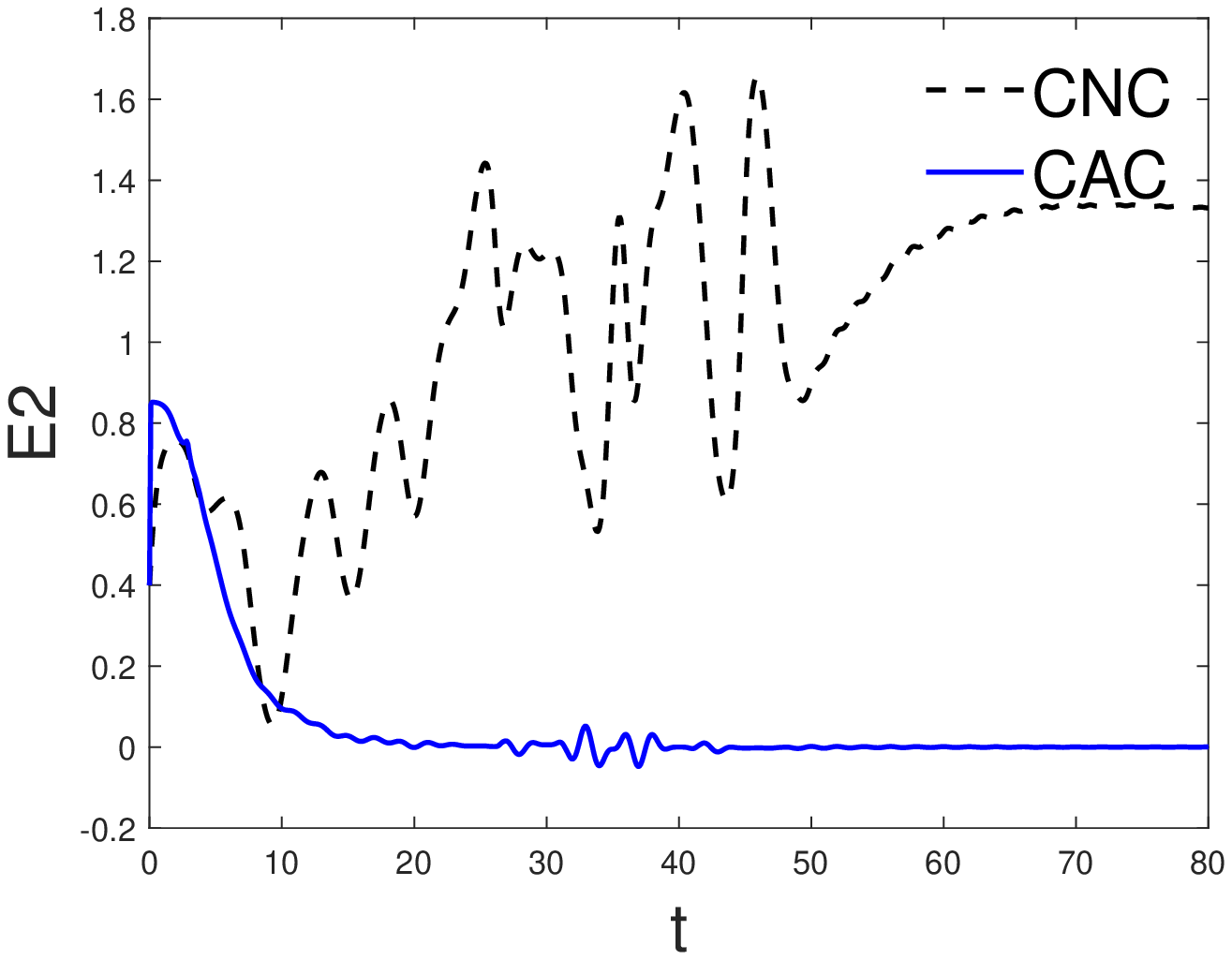}
    \end{psfrags}
  }
  \caption{Pose errors $e_1$ and $e_2$ for the NKC and AKC schemes.}
    \label{fig:posture_errors_prof01}
\end{figure}

Figure~\ref{fig:posture_errors_prof01} shows the pose error
$e={(e_1,e_2,e_3)}^T$. The dashed line stands for the NKC scheme, and
the solid line stands for the AKC scheme. As expected, the AKC scheme
achieves a significantly better performance compared to the NKC
scheme. Since the lateral slip $\sigma(t)$ and the longitudinal slip
derivatives $\dot{a}_l(t)$ and $\dot{a}_r(t)$ tend to zero on the
interval $t \geq 45 \ \mbox{s}$, the errors $e_1$, $e_2$ and $e_3$
tend to zero.

\begin{figure}[ht]
  \parbox{0.5\linewidth}{%
    \centering
    \begin{psfrags}
      \psfrag{t}[tc][c]{$t$ (s)}
      \psfrag{wLR}[c][c]{$\omega_l$ (rad/s)}
      \psfrag{CNC}[l][l]{\footnotesize NKC}
      \psfrag{CAC}[l][l]{\footnotesize AKC}
      \includegraphics[width=\linewidth]{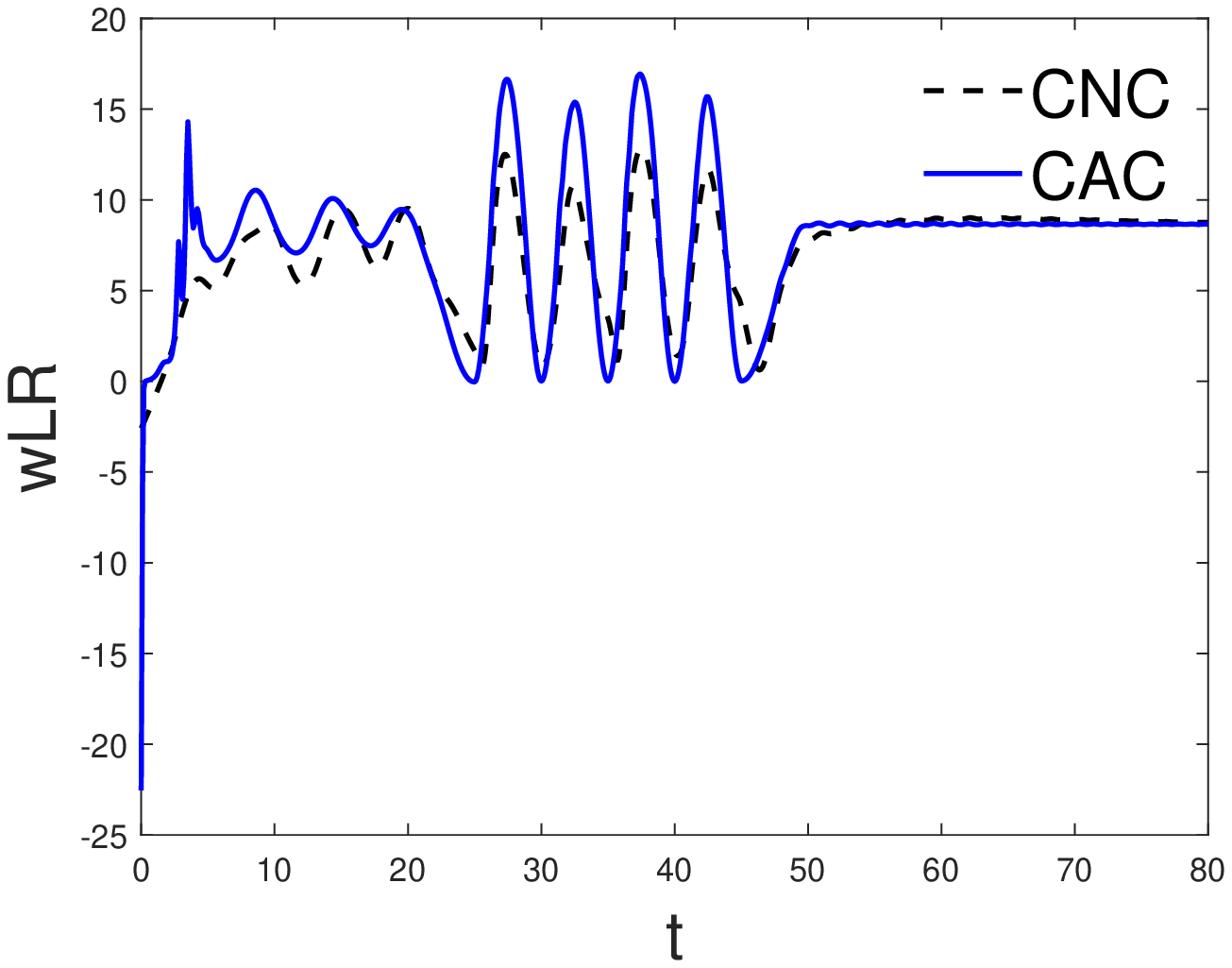}
    \end{psfrags}
  }%
  \parbox{0.5\linewidth}{%
    \centering
    \begin{psfrags}
      \psfrag{t}[tc][c]{$t$ (s)}
      \psfrag{wLR}[c][c]{$\omega_r$ (rad/s)}
      \psfrag{CNC}[l][l]{\footnotesize NKC}
      \psfrag{CAC}[l][l]{\footnotesize AKC}
      \includegraphics[width=\linewidth]{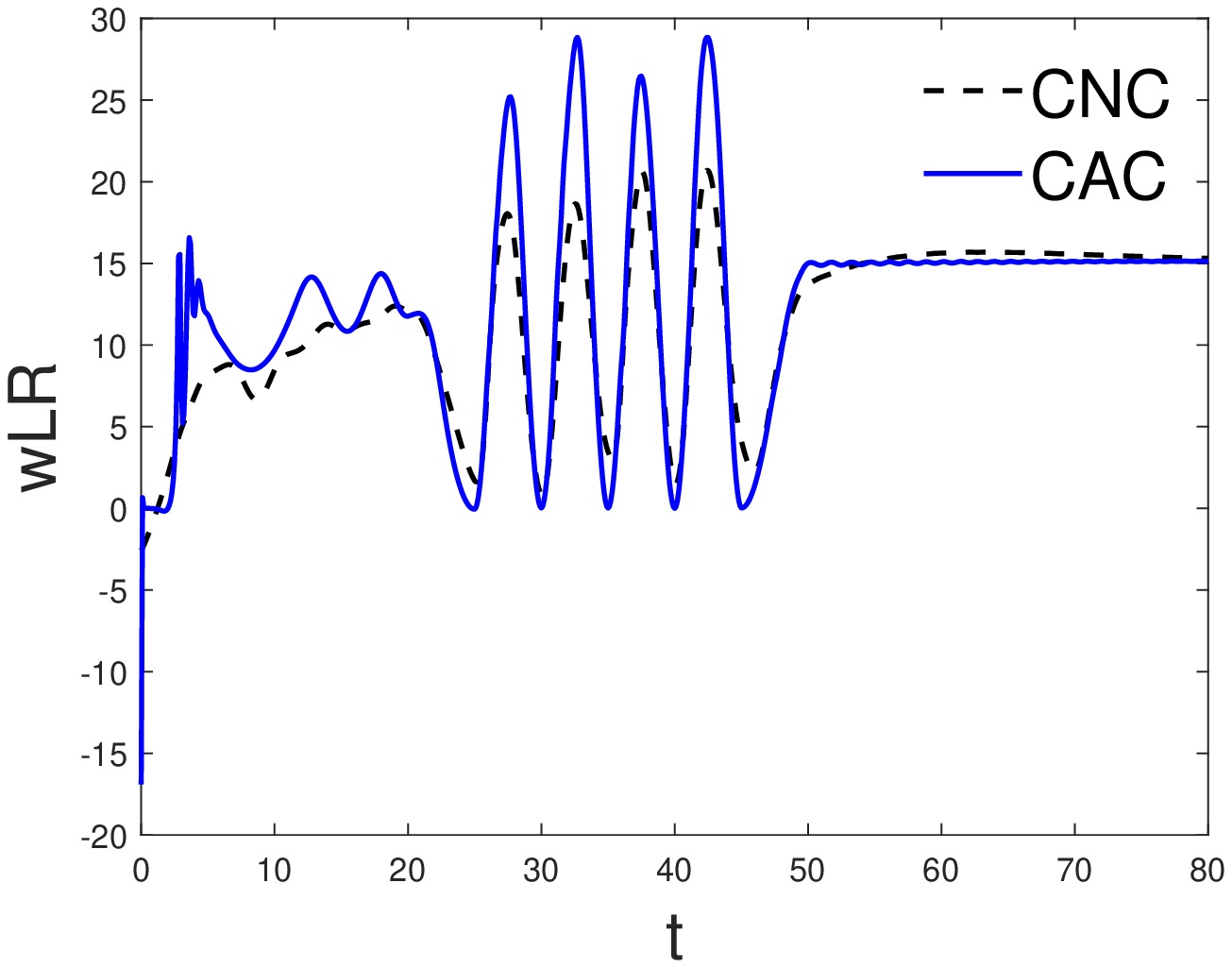}
    \end{psfrags}
  }
  \caption{The angular velocity of the left wheel $\omega_l(t)$
  and the angular velocity of the right wheel $\omega_r(t)$ using
  the NKC and AKC schemes.}
  \label{fig:effective_inputs_prof01}
\end{figure}

Figure~\ref{fig:effective_inputs_prof01} shows the effective control
input $\xi(t)={(\omega_l(t),\omega_r(t))}^T$ generated by the NKC and
AKC schemes. The dashed line stands for the NKC scheme, and the solid
line stands for the AKC scheme. Since, for the AKC scheme, the
auxiliary control input $\eta_c(t)$ converges to the reference input
$\eta_{\rm{ref}}(t)$, the amplitude of the effective control input
$\xi(t)$ is ultimately determined by the amplitude of the reference
input $\eta_{\rm{ref}}(t)$.

\section{Conclusions}
This paper provides a nonlinear kinematic controller that guarantees
reference tracking for a nonholonomic wheeled mobile robot subject to
wheel slip, using only the robot pose. The paper introduces the design
by defining the dynamics of the pose error including both longitudinal
and lateral slip. The proposed design builds on a kinematic control
law which is known to provide global trajectory tracking for wheeled
mobile robots in the absence of slip. Furthermore, this design is
extended to include an adaptive control law that compensates for the
slip. Our controller is guaranteed to provide asymptotic convergence
of the tracking error under time-varying references, when the robot is
only subject to constant longitudinal slip.

Moreover, it is rigorously shown that the tracking error and the
longitudinal slip estimation error are ultimately bounded provided
that the reference trajectory and longitudinal slip are slowly
varying, and the lateral slip is sufficiently small. To prove
convergence of the tracking error, the error dynamics is decomposed as
the sum of a nonlinear nominal term plus vanishing and nonvanishing
perturbation terms. Exponential stability of the nonlinear nominal
system is proved by applying stability criteria for slow time-varying
linear systems to the linearized model. Subsequently, the stability
under the presence of the perturbation terms are analyzed. Numerical
simulations show the effectiveness of the proposed approach.

\section*{Disclosure statement}
The authors declare that they have no 
conflict of interest.

\section*{Funding}
This work was supported by the Brazilian
funding agencies CAPES (Finance Code
001) and CNPq (Grant Numbers
503343/2012-9 and 311842/2013-5).

\bibliography{bibexport}

\appendix
\clearpage
\section{Supporting results}
\label{appx:stabilityperturbed}

\subsection{Ultimate boundedness}

Consider the system
\begin{equation}\label{appx:eq:perturbed_system_vanishing}
  \dot{x}(t) = f(t,x(t)) + g(t,x(t))
\end{equation}
where $f:[ 0,\infty ) \times D\to \mathbb{R}^n$ and
$g:[ 0,\infty ) \times D\to \mathbb{R}^n$ are piecewise
continuous in $t$ and locally Lipschitz in $x$ on $[ 0,\infty)
\times D$, and $D\subset\mathbb{R}^n$ is a domain that contains the
origin $x=0$. Consider system \eqref{appx:eq:perturbed_system_vanishing}
as a perturbation of the nominal system
\begin{equation}\label{appx:eq:nominal_system_vanishing}
  \dot{x}(t) = f(t,x(t))
\end{equation}

Suppose $x=0$ is an exponentially stable equilibrium point of the
nominal system \eqref{appx:eq:nominal_system_vanishing}, and let
$V(t,x(t))$ be a Lyapunov function that satisfies
\begin{equation}\label{appx:eq:contraint_01}
  c_1 \|x(t)\|^2 \leq V(t,x(t)) \leq c_2 \|x(t)\|^2
\end{equation}
\begin{equation}\label{appx:eq:contraint_02}
  \frac{\partial V(t,x(t))}{\partial t}
  +\frac{\partial V(t,x(t))}{\partial x(t)}f(t,x(t))\leq-c_3\|x(t)\|^2
\end{equation}
\begin{equation}\label{appx:eq:contraint_03}
  \left\| \frac{\partial V(t,x(t))}{\partial x(t)}\right\|
  \leq c_4\|x(t)\|
\end{equation}
for all $(t,x(t)) \in [ 0,\infty) \times D$ for some positive
constants $c_1$, $c_2$, $c_3$, and $c_4$.

For a vanishing perturbation $g(t,0)=0$, we can show that $x=0$ is
also an exponentially stable equilibrium point of the perturbed system
\eqref{appx:eq:perturbed_system_vanishing} using the next lemma.

\begin{lemma}[Lemma 9.1 from \cite{Khalil:2002:NS}]\label{lem:khalil91}
  Let $x=0$ be an exponentially stable equilibrium point of the
  nominal system \eqref{appx:eq:nominal_system_vanishing}. Let
  $V(t,x(t))$ be a Lyapunov function of the nominal system
  that satisfies \eqref{appx:eq:contraint_01}
  through \eqref{appx:eq:contraint_03} in $[0,\infty)\times D$.
  Suppose the perturbation term $g(t,x(t))$ satisfies
  \begin{equation}\label{appx:eq:contraint_04_vanishing}
    \|g(t,x(t))\|\leq \gamma \|x(t)\|, \ \forall \ t \geq0,
    \ \forall \ x(t) \in D
  \end{equation}
  and
  \begin{equation}\label{appx:eq:contraint_05_vanishing}
    \gamma < \displaystyle\frac{c_3}{c_4}
  \end{equation}
  Then, the origin is an exponentially stable equilibrium point
  of the perturbed system \eqref{appx:eq:perturbed_system_vanishing}.
\end{lemma}

For the more general case, in which it is not know whether 
$g(t,0)=0$, we can no longer study stability of the origin as an
equilibrium point, nor should we expect the solution of the
perturbed system to approach the origin as $t \to \infty$.
However, we can show that the solution is ultimately bounded by a
small bound using the next lemma.

\begin{lemma}[Lemma 9.2 from \cite{Khalil:2002:NS}]\label{lem:khalil92}
  Let $x=0$ be an exponentially stable equilibrium point of the
  nominal system \eqref{appx:eq:nominal_system_vanishing}. Let
  $V(t,x(t))$ be a Lyapunov function of the nominal system
  that satisfies \eqref{appx:eq:contraint_01}
  through \eqref{appx:eq:contraint_03} in $[0,\infty)\times D$, where
  $D=\{x(t) \in \mathbb{R}^n \,|\, \|x(t)\|<d\}$.
  Suppose the perturbation term $g(t,x(t))$ satisfies
  \begin{equation}\label{appx:eq:contraint_04_nonvanishing}
    \|g(t,x(t))\|\leq \delta < \frac{c_3}{c_4}\sqrt{\frac{c_1}{c_2}}
    \theta d
  \end{equation}
  for all $t\geq0$, all $x(t) \in D$, and some positive constant
  $\theta<1$. Then, for all $\|x(t_0)\|<\sqrt{c_1/c_2}d$, the solution
  $x(t)$ of the perturbed system \eqref{appx:eq:perturbed_system_vanishing}
  satisfies 
  \begin{equation*}
    \|x(t)\|\leq \alpha \exp\left[-\gamma(t-t_0)\right]\|x(t_0)\|,
    \quad \forall \ t_0 \leq t<t_0+T
  \end{equation*}
  and
  \begin{equation*}
    \|x(t)\|\leq u_b, \quad \forall \ t\geq t_0+T
  \end{equation*}
  for some finite $T$, 
  where
  \begin{equation*}
    \alpha = \sqrt{\frac{c_2}{c_1}}, \quad \gamma =
    \displaystyle\frac{(1-\theta)c_3}{2c_2}, \quad u_b =
    \displaystyle\frac{c_4}{c_3}\sqrt{\displaystyle\frac{c_2}{c_1}}
    \displaystyle\frac{\delta}{\theta}
  \end{equation*}
\end{lemma}

\subsection{Linearization}

To show that the origin $x=0$ of a nonlinear system is locally exponentially 
stable, the following well-know result is used:
\begin{theorem}[Theorem~4.15 from \cite{Khalil:2002:NS}]
	\label{theo:khalil4133ed}  
  Let $x(t)=0$ be an equilibrium point for the nonlinear system
  \begin{equation*}
    \dot{x}(t)=f(t,x(t))
  \end{equation*}
  where $f:[0,\infty) \times D \to \mathbb{R}^n$ is continuously
    differentiable, $D=\{x(t) \in \mathbb{R}^n \ | \ \|x(t)\|_2<d\}$, and
    the Jacobian matrix $[\partial f(t,x(t))/\partial x(t)]$ is bounded and
    Lipschitz on $D$, uniformly in $t$. Let
    \begin{equation*}
      A(t)=\displaystyle\frac{\partial f(t,x(t))}{\partial x(t)}
      \Bigg|_{x(t)=0}
    \end{equation*}
    Then, $x(t)=0$ is an exponentially stable equilibrium
    point
    for the nonlinear system if and only if it is an
    exponentially stable equilibrium point for the linear
    system
    \begin{equation*}
      \dot{x}(t) = A(t) x(t)
    \end{equation*}
\end{theorem}

Notice that exponential stability in \ref{theo:khalil4133ed} is uniform, 
see \cite[Definition 4.5]{Khalil:2002:NS}.

To show that the origin of a linear time-varying system is exponentially 
stable, the next theorem can be used:
\begin{theorem}[Theorem 1 of \cite{Rosenbrock:1963:SLT}]
	\label{theo:rosenbrock63}
  Let $\dot{x}(t)=A(t) x(t)$, where for all $t \ge t_0$ every
  element $a_{ij}(t)$ of $A(t)$ is differentiable and satisfies
  $|a_{ij}(t)|\leq c$ and every eigenvalue $\lambda$ of
  $A(t)$ satisfies
  \begin{equation*}
    \mbox{Re}[\lambda(A(t))] \leq -\epsilon <0
  \end{equation*}
  Then, there is some $\epsilon_d>0$ (independent of $t$) such that if
  every $|\dot{a}_{ij}|\leq \epsilon_d$, the equilibrium point
  $x(t)=0$ is asymptotically stable.
\end{theorem}

Notice that although the above result, which is reproduced from
\cite{Rosenbrock:1963:SLT}, concludes \textit{asymptotic} stability 
of the origin, the proof found in \cite{Rosenbrock:1963:SLT} shows that
\textit{exponential} stability also holds.

\section{Stability criterion of Li\'{e}nard and Chipart}\label{appx:LienardChipart}

\begin{theorem}[Theorem 11 from \cite{Gantmacher:1959:TM}]\label{theo:lienard}
  Necessary and sufficient conditions for all the roots of the real 
  polynomial
  \begin{equation*}
    p(s) = \alpha_0 s^n + \alpha_1 s^{n-1} + \dotsb + \alpha_n \quad
    (\alpha_0 > 0)
  \end{equation*}
  to have negative real parts can be given in any one of the following 
  four forms:
  \begin{enumerate}
  \item $\alpha_n>0,\alpha_{n-2}>0,\dotsc$; $\Delta_1 > 0,\Delta_3 > 0,\dotsc$,
  \item $\alpha_n>0,\alpha_{n-2}>0,\dotsc$; $\Delta_2 > 0,\Delta_4 > 0,\dotsc$,
  \item $\alpha_n>0,\alpha_{n-1}>0,\alpha_{n-3}>0,\dotsc$; 
    $\Delta_1 > 0,\Delta_3 > 0,\dotsc$,
  \item $\alpha_n>0,\alpha_{n-1}>0,\alpha_{n-3}>0,\dotsc$; 
    $\Delta_2 > 0,\Delta_4 > 0,\dotsc$,
  \end{enumerate}
  where
  \begin{equation*}
    \Delta_i = \det\left[
      \begin{array}{cccccc}
	\alpha_1 & \alpha_3 & \alpha_5 & \dotsb &  &  \\ 
	\alpha_0 & \alpha_2 & \alpha_4 & \dotsb &  &  \\ 
	0 & \alpha_1 & \alpha_3 & \dotsb &  &  \\ 
	0 & \alpha_0 & \alpha_2 & \alpha_4 &  &  \\ 
	&  &  &  & \ddots &  \\ 
	&  &  &  &  & \alpha_i
      \end{array} \right] \qquad (\text{$\alpha_k = 0$\  for  $k>n$})
  \end{equation*}
\end{theorem}
is the Hurwitz determinant of order $i$ $(i=1,2,\dotsc,n)$.

\begin{remark}\label{rem:strictly}
  In this paper, the main application of above
  Theorem~\ref{theo:lienard} is concerned with the stability of matrix
  $A(t)$, which is time varying. Thus, the coefficient of its
  characteristic polynomial
  \begin{equation*}
    p(s) = s I - A(t) = \alpha_0(t) s^n + \alpha_1(t) s^{n-1} + \dotsb
    + \alpha_n(t)
  \end{equation*}
  will also be a function of time.

  However, one of the coefficient $\alpha_i(t)$, although been
  strictly positive at any instant of time, $\alpha_i(t) > 0$, might
  have the property that $\lim_{t\to\infty} \alpha_i(t) = 0$, and
  consequently
  \begin{equation*}
    \lim_{t\to\infty} \mbox{Re}[\lambda_j(A(t))] = 0    
  \end{equation*}
  for some eigenvalue $\lambda_j$. To overcome this fact in order to
  impose the eigenvalues are bounded away from the imaginary axis by
  an amount $\epsilon>0$, such that
  \begin{equation*}    
    \lambda_j(A(t)) < -\epsilon < 0
  \end{equation*}
  it is necessary to impose a constant lower bound $\nu >0$ on all
  the conditions in Theorem~\ref{theo:lienard}. For instance, the
  first condition must be replaced by
  \begin{enumerate}
  \item $\alpha_n> \nu > 0,\alpha_{n-2} > \nu > 0,\dotsc$;
    $\Delta_1 > \nu > 0,\Delta_3 > \nu > 0,\dotsc$,
  \end{enumerate} 
\end{remark}

\section{Characteristic polynomial}\label{appx:poly}

The characteristic polynomial of matrix $A(t)$, defined in
\eqref{eq:linearization_matrix}, is
\begin{equation*}
  p(s) = s^5 + \alpha_1(t) s^4 + \alpha_2(t) s^3 + \alpha_3(t) s^2 +
  \alpha_4(t) s + \alpha_5(t)
\end{equation*}
with the coefficients $\alpha_i(t)$ given by
\begin{equation*}
  \begin{split}
    \alpha_1(t) &= k_1 + \displaystyle \frac{k_4}{2k_3} v_{\rm{ref}}
    \\[3pt]
    \alpha_2(t) &= \frac{1}{16 a_l a_r b^2 k_2 k_3}
    \Big\{ a_l \left[ \gamma_2 v_1^2 k_3 (b^2 k_2 + 4 k_4)
      + 8 a_r b^2 k_2 \left(
      v_{\rm{ref}}( k_1 k_4 + k_2 k_3 v_{\rm{ref}})
      + 2 k_3 \omega_{\rm{ref}}^2 \right) \right]
    \\
    & \qquad + a_r \gamma_1 v_2^2 k_3 (b^2 k_2 + 4 k_4)  \Big\}
    \\[3pt]
    \alpha_3(t) &= \frac{1}{32 a_l a_r b^2 k_2 k_3}
    \bigg\{ a_r \Big[ 16 a_l b^2 k_2 v_{\rm{ref}}
      (k_1 k_2 k_3 v_{\rm{ref}} + \omega_{\rm{ref}}^2)
      \\
      & \qquad + \gamma_1 v_2^2 \left(k_2
      v_{\rm{ref}} \left( k_3^2 (b^2 k_2 + 8) + b^2 \right) + 8 k_1 k_3 k_4
      \right) \Big]
    \\
    & \qquad + a_l \gamma_2 v_1^2 \left[ k_2 v_{\rm{ref}}
      \left( k_3^2 (b^2 k_2 + 8) + b^2 \right) + 8 k_1 k_3 k_4 \right] \bigg\}
    \\[3pt]
    \alpha_4(t) &= \frac{1}{32 a_l a_r b^2 k_2}
    \Big\{ \gamma_2 v_1^2 \left[ a_l
      \left( {(b k_2 v_{\rm{ref}} + 2 \omega_{\rm{ref}})}^2
      +4 \omega_{\rm{ref}}^2 + 8 k_1 k_2 k_3 v_{\rm{ref}} \right)
      + 2 v_2^2 \gamma_1 k_4  \right]
    \\
    & \qquad + a_r \gamma_1 v_2^2
    \left( {(b k_2 v_{\rm{ref}} - 2 \omega_{\rm{ref}})}^2
    + 4 \omega_{\rm{ref}}^2 + 8 k_1 k_2 k_3 v_{\rm{ref}} \right) \Big\}
    \\[3pt]
    \alpha_5(t) &= \displaystyle \frac{1}{16 a_l a_r b^2}
    \gamma_1 \gamma_2 k_3 v_{\rm{ref}} v_1^2 v_2^2
  \end{split}
\end{equation*}
with $v_1 = b \omega_{\rm{ref}} + 2 v_{\rm{ref}}$, $v_2 = b
\omega_{\rm{ref}} -2 v_{\rm{ref}}$, and $k_4 = 1 + k_2 k_3^2$.

\section{Reference trajectory}\label{appx:reference}

The reference trajectory used in the simulations is generated by the
numerical integration of the kinematic model
\eqref{eq:matrix_robot_reference} with the initial condition
$q_{\rm{ref}}(0)={(0,0,0)}^T$ and the reference input
$\eta_{\rm{ref}}(t)={(v_{\rm{ref}}(t),\omega_{\rm{ref}}(t))}^T$, adapted from
\cite{Kanayama:1990:STC}, given by
\begin{equation*}
  \begin{aligned}
    0 \ \mbox{s} \le t < 5 \ \mbox{s}:
    & \quad v_{\rm{ref}}(t) = 0.25\left(1-\cos\left(\frac{\pi t}{5}\right)\right)
    \quad &\mbox{and} \quad &\omega_{\rm{ref}}(t) = 0 \\[-7pt]
    5 \ \mbox{s} \le t < 20 \ \mbox{s}:
    & \quad v_{\rm{ref}}(t) = 0.5 \quad &\mbox{and} \quad &\omega_{\rm{ref}}(t) = 0 \\
    20 \ \mbox{s} \le t < 25 \ \mbox{s}:
    & \quad v_{\rm{ref}}(t) = 0.25\left(1+\cos\left(\frac{\pi t}{5}\right)\right)
    \quad &\mbox{and} \quad &\omega_{\rm{ref}}(t) = 0 \\
    25 \ \mbox{s} \le t < 30 \ \mbox{s}:
    &\quad v_{\rm{ref}}(t) = 0.15\pi\left(1-\cos\left(\frac{2\pi t}{5}\right)\right)
    \quad &\mbox{and} \quad &\omega_{\rm{ref}}(t) = -v_{\rm{ref}}(t)/1.5 \\
    30 \ \mbox{s} \le t < 35 \ \mbox{s}:
    & \quad v_{\rm{ref}}(t) = 0.15\pi\left(1-\cos\left(\frac{2\pi t}{5}\right)\right)
    \quad &\mbox{and} \quad &\omega_{\rm{ref}}(t) = v_{\rm{ref}}(t)/1.5 \\
    35 \ \mbox{s} \le t < 40 \ \mbox{s}:
    &\quad v_{\rm{ref}}(t) = 0.15\pi\left(1-\cos\left(\frac{2\pi t}{5}\right)\right)
    \quad &\mbox{and} \quad &\omega_{\rm{ref}}(t) = -v_{\rm{ref}}(t)/1.5 \\
    40 \ \mbox{s} \le t < 45 \ \mbox{s}:
    & \quad v_{\rm{ref}}(t) = 0.15\pi\left(1-\cos\left(\frac{2\pi t}{5}\right)\right)
    \quad &\mbox{and} \quad &\omega_{\rm{ref}}(t) = v_{\rm{ref}}(t)/1.5 \\
    45 \ \mbox{s} \le t < 50 \ \mbox{s}:
    & \quad v_{\rm{ref}}(t) = 0.25\left(1+\cos\left(\frac{\pi t}{5}\right)\right)
    \quad &\mbox{and} \quad &\omega_{\rm{ref}}(t) = 0 \\[-5pt]
    50 \ \mbox{s} \le t:
    & \quad v_{\rm{ref}}(t) = 0.5 \quad &\mbox{and} \quad &\omega_{\rm{ref}}(t) = 0
  \end{aligned}
\end{equation*}

\end{document}